\documentclass[12pt]{article}
\usepackage[margin=1in]{geometry}

\usepackage{algorithm}
\usepackage{algorithmic}
\usepackage{booktabs} 
\usepackage{amsthm}
\usepackage{natbib}
\usepackage{amsmath}
\usepackage{mathtools}
\usepackage{amssymb}
\usepackage[utf8]{inputenc}
\usepackage{graphicx}
\usepackage{xcolor}
\usepackage{comment}
\usepackage[english]{babel}
\usepackage{bm}
\usepackage{multirow}
\usepackage{url}

\newtheorem{theorem}{Theorem}
\newtheorem{proposition}{Proposition}
\newtheorem{lemma}{Lemma}

\newtheorem{assumption}{Assumption}

\DeclarePairedDelimiter\floor{\lfloor}{\rfloor}
\DeclareMathOperator*{\argmin}{arg\,min}
\DeclareMathOperator*{\st}{s.t.}
\usepackage{subcaption}
\usepackage[labelformat=parens,labelsep=quad,skip=3pt]{caption}

\title{The Stochastic Proximal Distance Algorithm}

\author{
  Haoyu Jiang \\
  Department of Statistics, University of Illinois Urbana-Champaign \\
  Jason Xu \\
  Department of Statistical Science, Duke University
}
\date{}

\begin{document}
\maketitle

\begin{abstract}
Stochastic versions of proximal methods have gained much attention in statistics and machine learning. These algorithms tend to admit simple, scalable forms, and enjoy numerical stability via implicit updates. In this work, we propose and analyze a stochastic version of the recently proposed proximal distance algorithm, a class of iterative optimization methods that recover a desired constrained estimation problem as a penalty parameter $\rho \rightarrow \infty$. By uncovering connections to related stochastic proximal methods and interpreting the penalty parameter as the learning rate, we justify heuristics used in practical manifestations of the proximal distance method, establishing their convergence guarantees for the first time. Moreover, we extend recent theoretical devices to establish finite error bounds and a complete characterization of convergence rates regimes. We validate our analysis via a thorough empirical study, also showing that unsurprisingly, the proposed method outpaces batch versions on popular learning tasks.
\end{abstract}

\section{Introduction}\label{sec_intro}
Optimization of an objective function subject to constraints or regularization is ubiquitous in statistics and machine learning. 
Consider the canonical problem setting
\begin{equation}\label{obj_intro}
\min_{\bm{\theta}}\; F(\bm{\theta}) = \frac{1}{n} \sum_{i=1}^n f(\bm{\theta};\bm{z}_i) \qquad \st \bm{\theta} \in C 
\end{equation}
where $\{\bm{z}_i\}_{i=1}^n$ is the observed data, $\bm{\theta}$ is the model parameter, $f$ is typically a measure of fit, and $C$ is a set imposing desired structure on the solution. Toward solving constrained and regularized objectives, there are many variants of gradient methods popular in the literature. We survey a number of these in Section 2, which often employing projection or proximal steps to handle constraints and non-smooth penalties. Here we focus attention on a recent method by \citet{lange2015proximal} that is appealing in providing a broad framework to handle many constrained problems gracefully. Called the \textit{proximal distance algorithm}, their approach transforms \eqref{obj_intro} to the unconstrained problem 
\begin{equation}\label{eq:unc}\min_{\bm{\theta}} \; F(\bm{\theta}) + \frac{\rho}{2}\text{dist}(\bm{\theta},C)^2, \end{equation} where $\text{dist}(\bm{\theta},C)= \inf_{\bm{\theta}' \in C}\,\|\bm{\theta} - \bm{\theta}'\|_2$ denotes the Euclidean distance between $\bm{\theta}$ and the constraint set $C$. Using the idea of distance majorization \citep{chi2014distance}, this reformulation can be solved via iterative algorithms from the perspective of  Majorization-Minimization, or MM \citep{hunter2004tutorial, mairal2015incremental} as long as the projection operators are practical to evaluate. This is the case for many common constraints including sparsity, rank, and shape restrictions; a broad variety of examples are considered in \citet{NIPS2017_061412e4,keys2019proximal,landeros2022extensions}.

This proximal distance method can be viewed as a smooth version of the classical penalty method due to \citet{Courant1943VariationalMF}, but under a squared penalty as in \eqref{eq:unc}, the  unconstrained solutions  
recover the constrained solution only in the limit as $\rho\rightarrow \infty$ \citep{beltrami1973algorithmic}. Early attempts to set $\rho$ at a large value encountered slow progress, as the reciprocal $1/\rho$ plays the role of a step-size in the algorithm, discussed further below.  This is supported by recent analyses by \citet{keys2019proximal, landeros2022extensions} revealing a convergence rate in convex cases of $\mathcal{O}(\rho k^{-1})$, where $k$ is the iteration counter. In light of this observation, researchers have suggested gradually increasing $\rho$ within the iterative algorithm. \citet{keys2019proximal} suggests setting an initial penalty value $\rho=\rho_0$, rate $\gamma>1$ and frequency $m$, and then updating $\rho_k=\rho_0 \gamma^{\floor{\frac{k}{m}}}$, which increases roughly exponentially. While this is consistent with the intuition of approaching the constrained solution, the traditional MM geometry used to derive the algorithm---along with its corresponding guarantees--- does not apply to heuristics that entail a sequence of \textit{changing} objective functions \eqref{eq:unc} indexed by the increasing sequence $\rho_k$. These heuristics for increasing $\rho$ play an important role in practice so that iterates eventually respect the constraint when the penalty becomes large, but allow for large enough data-driven steps in the earlier stages of the algorithm.

Fortunately, there is a large literature on related stochastic proximal methods, which are increasingly prominent in statistical and machine learning problems involving large data sets \citep{toulis2015proximal, https://doi.org/10.48550/arxiv.2206.12663}. In this paper, we draw new connections between the proximal distance algorithm and these studies, leveraging techniques for studying their convergence toward rigorously resolving the open questions regarding the convergence of the proximal distance algorithm under various $\rho_k$ schedules. Our point of departure is to propose a stochastic version of the proximal distance algorithm. By evaluating the loss at only a subsample of the data, we break from the original geometry of majorization from which the method was derived. However, we take a different view by drawing analogies to  implicit gradient methods. Noting the relation between the penalty parameter and the step size or \textit{learning rate} in proximal/implicit algorithms, we establish connections to implicit SGD \citep{toulis2014statistical, 10.2307/26362880, https://doi.org/10.48550/arxiv.2206.12663}, the  proximal Robbins-Monro algorithm \citep{toulis2015proximal}, and incremental  proximal algorithms \citep{bertsekas2011incremental}. We provide new theoretical analyses that reveal convergence rates in several regimes under various polynomial learning rate schedules under weaker assumptions than previous studies. 

Our contributions bridge a theoretical gap in justifying  heuristics commonly used in practical implementations of the proximal distance algorithm. While the findings are largely theoretical in nature, the resulting stochastic algorithm naturally enables large-scale data to be analyzed that would be intractable under the full batch version of the method. The remainder of this manuscript is organized as follows: we begin with an overview and necessary background in Section 2. Next, Section 3 proposes our stochastic version of the proximal distance algorithm, and discusses connections and differences to related stochastic proximal methods. In Section 4, we present our theoretical analysis of the algorithm, establishing finite error bounds, revealing convergence rates and providing a rigorous justification for mechanisms to increase $\rho$. In Section 5, we conduct empirical study to validate our analysis and investigate behavior in practice for both convex and non-convex settings, while demonstrating computational gains over the full version. We conclude and discuss future directions in Section 6. 

\section{Background and Motivation} 
\paragraph{Proximal Algorithms}
Many powerful iterative algorithms for optimization involve the proximal operator \citep{bauschke2011convex, parikh2014proximal}, defined as
$$\text{prox}_f(\bm{y}) = \argmin_{\bm{x}} \;f(\bm{x}) + \frac{1}{2}\lVert \bm{x}-\bm{y} \rVert^2.$$
A fundamental example is the \textit{proximal point algorithm} \citep{rockafellar1976monotone, parikh2014proximal}, which minimizes an objective $F$ with successive proximal operations:
$$\bm{\theta}_{k+1} = \text{prox}_{\alpha F}(\bm{\theta}_k) = \argmin_{\bm{\theta}}\; F(\bm{\theta}) + \frac{1}{2\alpha}\|\bm{\theta} - \bm{\theta}_k\|^2.$$
Intuitively, this seeks to shrink toward the current iterate $\bm{\theta}_k$ when minimizing $F$, where a parameter $\alpha$ determines the strength of the shrinkage. Methods such as \textit{proximal gradient descent} \citep{parikh2014proximal, li2015accelerated} consider composite objectives, alternating between proximal operations and gradient descent steps.

Of particular interest here is the proximal distance algorithm \citep{lange2015proximal}. It extends the penalty method of \citet{Courant1943VariationalMF} which relaxes a constrained problem \eqref{obj_intro} 
to an unconstrained problem of the form $$\min_{\bm{\theta}}\; F(\bm{\theta}) + \rho q(\bm{\theta}),$$ 
where $q(\bm{\theta})\ge 0$ is a penalty that vanishes on the constraint set: $q(\bm{\theta})=0$ for all $\bm{\theta} \in C$. Solutions to this unconstrained problem $\bm{\theta}_*^\rho$ are guaranteed to converge to the constrained solution as $\rho\rightarrow \infty$ \citep{beltrami1973algorithmic}. The proximal distance algorithm combines this idea with distance majorization \citep{chi2014distance}, taking the penalty $q$ to be the squared distance between $\bm{\theta}$ and the constraint set $C$.  A key algorithmic component is the set projection: note the squared distance between $\bm{\theta}$ and $C$ can also be expressed as
$$q(\bm{\theta})=\text{dist}(\bm{\theta}, C)^2 = \|\bm{\theta} - P_C(\bm{\theta})\|^2,$$ 
where $P_C(\bm{\theta})=\argmin_{\bm{\theta}'\in C}\,\|\bm{\theta}' - \bm{\theta}\|_2$ denotes the projection of $\bm{\theta}$ on $C$. \citet{lange2015proximal} utilize this to convert the constrained problem (\ref{obj_intro}) to an unconstrained form:
\begin{equation}\label{obj_unconstrained}
\min_{\bm{\theta}}\; F(\bm{\theta}) + \frac{\rho}{2}\|\bm{\theta} - P_C(\bm{\theta})\|^2.
\end{equation}
Though \eqref{obj_unconstrained} is often not directly solvable, the proximal distance algorithm makes use of majorization-minimization (MM) to efficiently solve a sequence of simpler subproblems.
In more detail, distance majorization \citep{chi2014distance} implies that the following \textit{surrogate function}
\begin{equation}\label{surrogate_func}
F(\bm{\theta}) + \frac{\rho}{2}\lVert \bm{\theta} - P_C(\bm{\theta}_{k-1}) \rVert^2
\end{equation}
majorizes the expression in \eqref{obj_unconstrained}. The principle of MM suggests then minimizing the surrogate (\ref{surrogate_func}):
\begin{equation}\label{algo_fixrho}
\bm{\theta}_{k} = \argmin_{\bm{\theta}}\; F(\bm{\theta}) + \frac{\rho}{2}\lVert \bm{\theta} - P_C(\bm{\theta}_{k-1}) \rVert^2 
= \text{prox}_{\rho^{-1}F}[P_C(\bm{\theta}_{k-1})].
\end{equation}
The resulting MM iteration consists of alternatively defining surrogates and minimizing them according to these updates. By ensuring a descent property, it is easy to show that the iterate sequence $\bm{\theta}_k$ converges to a local minimum $\bm{\theta}_*^\rho$ of (\ref{obj_unconstrained}). In convex settings, $\bm{\theta}_*^\rho$ is the global minimizer.

In practice, convergence can be slow when fixing a large value for $\rho$, a pitfall past approaches seek to remedy by employing an increasing sequence of penalty parameters $\{\rho_k\}$. When $\rho$ is small, the objective is dominated by the loss $F$ which drives the MM updates; as $\rho$ grows larger, the penalty plays a stronger role, guiding $\{\bm{\theta}_k\}$ toward the constraint set. Heuristically increasing $\rho$ at an effective rate plays a key role in the practical performance of the method, which we will investigate more closely below.

\paragraph{Stochastic Proximal Methods}\citet{robbins1951stochastic} proposed a stochastic approximation method for finding the roots of functions that lays the foundation for many stochastic optimization methods used in large scale statistical and machine learning problems. A well-known instance is stochastic gradient descent (SGD) \citep{nemirovski2009robust, moulines2011non,bietti2017stochastic}, a method for smooth unconstrained optimization that combines the Robbins-Monro algorithm with gradient descent in using subsamples toward approximate steps. It updates
$$\bm{\theta}_{k+1} = \bm{\theta}_k - \alpha_k \nabla f(\bm{\theta}_k, \bm{z}_{\xi_k}),$$
where $\bm{z}_{\xi_k}$ denotes a randomly drawn observation from the dataset and $\alpha_k$ is the sequence of learning rates. For convex cases, $\bm{\theta}_k$ converges in probability to the solution whenever
\begin{equation*}
\sum_{k=1}^{\infty} \alpha_k = \infty, \quad \sum_{k=1}^{\infty} \alpha_k^2 < \infty.
\end{equation*}
While SGD enables many learning tasks in the large-scale and online settings, it can be sensitive to initial learning rate, with small learning rates effecting slow convergence and large learning rates risking divergence  \citep{moulines2011non, ryu2014stochastic}.  
Stochastic versions of proximal methods have been studied to improve numerical instability. The \textit{stochastic proximal point algorithm} \citep{bianchi2016ergodic, asi2019stochastic}, also called  \textit{implicit SGD} \citep{toulis2014statistical, 10.2307/26362880, https://doi.org/10.48550/arxiv.2206.12663}, applies stochastic approximation to the proximal point algorithm. Subsampling one data point at each iteration,
\begin{equation}\label{implicit_sgd}
\bm{\theta}_{k+1} = \argmin_{\bm{\theta}} f(\bm{\theta};\bm{z}_{\xi_k}) + \frac{1}{2\alpha_k} \lVert \bm{\theta} - \bm{\theta}_k \rVert^2 
= \bm{\theta}_k - \alpha_k \nabla f(\bm{\theta}_{k+1};\bm{z}_{\xi_k}),
\end{equation}
where the second equation is obtained by equating the gradient to zero. Note $\bm{\theta}_{k+1}$ appears on both sides; for this reason, the method is referred to as \textit{implicit}. By virtue of the proximal operation, this implicit iteration is more stable than (explicit) SGD with respect to the choice of initial step size in theory and practice  \citep{ryu2014stochastic, toulis2014statistical}. \citet{toulis2015proximal} also consider an idealized implicit SGD in the setting where $\xi_k$ itself is not observable. They incorporate a nested SGD as a subroutine and show that the method still inherits desirable stability properties. Similar  \textit{stochastic proximal gradient algorithms} \citep{nitanda2014stochastic} subsample in the gradient descent step, and further employ Monte Carlo methods to approximate the gradient when it is analytically intractable \citep{atchade2017perturbed}. Finally, the incremental proximal method of \citet{bertsekas2011incremental} is related to stochastic proximal gradient, but performs an incremental version of both the gradient descent step and the proximal step, and focuses on constrained problems.

\section{A Stochastic Proximal Distance Algorithm}\label{our_method}
Motivated by the recent success of stochastic proximal methods in machine learning problems on massive datasets, it is natural to now consider a stochastic version of the proximal distance algorithm. Doing so allows the method to scale to problem sizes that were previously infeasible, and to apply to online data settings. Scaling the method is not the focus of our study, however; our analysis yields more surprising results on the convergence rates properties and properties of various penalty parameter schedules. This sheds light on the precise role of $\rho_k$ not known in previous studies---even in the full batch setting, convergence guarantees were lacking. 

For intuition, we first consider the case that $F$ is convex. Taking the gradient of the surrogate function \eqref{surrogate_func} and equating to $0$ yields the stationarity equation 
$$\nabla F(\bm{\theta}) + \rho_k [\bm{\theta} - P_C(\bm{\theta}_{k-1})] = 0,$$
which furnishes the next iterate $\bm{\theta}_{k}$ as the minimizer:
\begin{equation}\label{imp_update_form}
    \bm{\theta}_{k} = P_C(\bm{\theta}_{k-1}) - \frac{1}{\rho_k}\nabla F(\bm{\theta}_{k}).
\end{equation}
We see from (\ref{imp_update_form}) that the update resembles an \textit{implicit gradient} step but with an additional projection inside. Note that $1/\rho_k$ plays the role of the step size or \textit{learning rate} in this analogy. Bridging to the intuition behind methods in Section 2, this leads to a natural stochastic variant of the proximal distance algorithm by  approximating the implicit gradient using only a subset of the data. Doing so yields the following \textit{stochastic proximal distance iteration}:
\begin{equation}\label{imp_update_1st}
\begin{split}
    \bm{\theta}_k &= P_C(\bm{\theta}_{k-1}) - \frac{1}{\rho_k}\nabla f(\bm{\theta}_k;\bm{z}_{\xi_k}) \\
    &= \argmin_{\bm{\theta}}\; f(\bm{\theta};\bm{z}_{\xi_k}) + \frac{\rho_k}{2} \|\bm{\theta} - P_C(\bm{\theta}_{k-1})\|^2 \\
    &= \text{prox}_{\rho_k^{-1}f(\cdot;\bm{z}_{\xi_k})}[P_C(\bm{\theta}_{k-1})],
\end{split}
\end{equation}
where the index $\xi_k$ is drawn uniformly from $\{1,2,...,n\}$. Here we have considered the subsample to be a single point for exposition, though the idea applies immediately to minibatches of size $b$ by sampling a set of indices $\mathcal{I}_k$ without replacement with $\lvert \mathcal{I}_k \rvert = b$. In this case, $f(\bm{\theta};\bm{z}_{\xi_k})$ in (\ref{imp_update_1st}) is replaced by $\frac{1}{b}\sum_{i\in\mathcal{I}_k} f(\bm{\theta};\bm{z}_i)$; Algorithm \ref{stoalgo} summarizes the procedure in pseudocode.

\begin{algorithm}[tb]
\caption{Stochastic proximal distance algorithm}\label{stoalgo}
\begin{algorithmic}
   \STATE {\bfseries input:} data $\{\bm{z}_i\}_{i=1}^n$, 
initial penalty $\rho_1$, rate $\gamma$, batch size $b$, initial guess $\bm{\theta}_0$, iteration cap $K_{\text{max}}$, tolerance $\epsilon$
   \STATE {\bfseries output:} $\hat{\bm{\theta}}\approx \argmin_{\bm{\theta} \in C}\;\frac{1}{n}\sum_{i=1}^n f(\bm{\theta};\bm{z}_i)$
   
   Initialize $k=0$
   \REPEAT
   \STATE $k \gets k+1,$\,  $\rho \gets \rho_1 \times k^{\gamma}$;
   \STATE $\mathcal{I} \gets$ $b$ samples taken from $\{1,...,n\}$ w/o replacement;
   \STATE $\bm{\theta}_k  \gets \text{prox}_{\rho^{-1} b^{-1}\sum_{i \in \mathcal{I}} f(\cdot;\bm{z}_i)}[P_C(\bm{\theta}_{k-1})]$
   \UNTIL{$\lvert F[P_C(\bm{\theta}_k)] - F[P_C(\bm{\theta}_{k-1})] \rvert < \epsilon$ or $k=K_{\text{max}}$}
   \STATE $\hat{\bm{\theta}} \gets P_C(\bm{\theta}_k)$
\end{algorithmic}
\end{algorithm}

It is worth noting that although the connection to implicit gradient descent provides a useful perspective to derive an analogy, our algorithm does \textit{not} require differentiability in practice. Algorithm 1 requires only that the projection and the proximal mapping are computable. This is the case for many common constraints (see Section \ref{sec_intro}), and a convex loss component suffices to guarantee the existence of the proximal mapping of $P_C(\bm{\theta}_{k-1})$. After convergence, we advocate one final projection step to ensure constraints are enforced exactly, as denoted in the last line of the pseudocode. 

\paragraph{Intuition and Convergence}
The proximal distance algorithm is originally derived from the perspective of MM. From this view, the stochastic proximal distance algorithm (\ref{imp_update_1st}) substitutes the loss at one data point to approximate the average loss across all the data when building the surrogate function. Under this approximation, however, majorization of the objective no longer holds. Thus, the overall geometry and descent property from the MM algorithm it modifies no longer carry through. Our intuition tells us that while the loss evaluated on a subsample fluctuates, it has mean equal to the objective and should tend to decrease the loss, though we do not have a precise theory of MM in expectation here. 

We may also view the stochastic proximal distance algorithm (\ref{imp_update_1st}) from the lens of gradient descent. At each iteration, we are approximating a true direction of steepest descent given the data, $\nabla F(\bm{\theta}_k^+)$, by a noisy direction $\nabla f(\bm{\theta}_k;\bm{z}_{\xi_k})$ based on only a subsample from the data.  Here $\bm{\theta}_k^+$ denotes the update from the full non-stochastic or \textit{batch}  update: that is, $\bm{\theta}_k^+ = \text{prox}_{\rho_k^{-1}F}[P_C(\bm{\theta}_{k-1})]$. It is worth pointing out that unlike explicit gradient methods, the gradient $\nabla f(\bm{\theta}_k;\bm{z}_{\xi_k})$ is no longer an unbiased estimate of the true direction (this is also true for implicit SGD). Denote
$ \displaystyle \bm{\theta}_k^i = \text{prox}_{\rho_k^{-1}f(\cdot;\bm{z}_i)}[P_C(\bm{\theta}_{k-1})]$ for all $i=1,...,n$:
then denoting the natural filtration $\mathcal{F}_{k-1} = \sigma(\xi_1,...,\xi_{k-1})$,
\begin{equation*}
\mathbb{E}(\nabla f(\bm{\theta}_k;\bm{z}_{\xi_k})|\mathcal{F}_{k-1})=\frac{1}{n}\sum_{i=1}^n \nabla f(\bm{\theta}_k^i;\bm{z}_i)  
\ne  \nabla F(\bm{\theta}_k^+).
\end{equation*}
While we expect these to yield a descent direction on average, the right hand side emphasizes that the noisy direction is not necessarily unbiased as  is the case in implicit update schemes. This motivates a rigorous analysis of the convergence properties, presented in Section 4. 
 
Following the intuition in previous work on proximal distance algorithms, the penalty parameter sequence $\rho_k$ must increase to approach the constrained solution. This agrees with viewing our method from the lens of stochastic gradient methods, where a diminishing learning rate is necessary to decrease the variance of the noisy direction  to ensure convergence. We see then that here $\rho_k$ plays a dual role as a step-size algorithmically, and as a parameter to modulate the transition from measure of fit to constraint term. This key observation, which we expand upon immediately below, leads us to import theoretical tools used to analyze the convergence of stochastic approximation algorithms to better understand the sequence of proximal distance objectives, for which only heuristic justifications are previously available \citep{xu2021proximal,landeros2022extensions}. 

\paragraph{Connection to Incremental Proximal Methods}
Though the prior discussion does not rigorously establish convergence, we find a surprising connection between the stochastic algorithm and incremental proximal methods. Indeed, despite deriving from an entirely different perspective, we found Algorithm 1 shares an \textit{identical update rule} with the incremental proximal algorithm of \citet{bertsekas2011incremental}. There, the authors consider the problem
$$\min_{\bm{\theta}_\in C}\; \sum_{i=1}^n f(\bm{\theta};\bm{z}_i) + g(\bm{\theta};\bm{z}_i),$$
and derive updates consisting of a combination of an incremental proximal step and an incremental gradient step:
\begin{equation}\label{incre_prox_alg}
\begin{split}
\bm{\eta}_k &= P_C[\bm{\theta}_k - \alpha_k \nabla f(\bm{\eta}_k;\bm{z}_{\xi_k})], \\
\bm{\theta}_{k+1} &= P_C[\bm{\eta}_k - \alpha_k \nabla g(\bm{\eta}_k;\bm{z}_{\xi_k})].
\end{split}
\end{equation}
The authors show that the order of these two steps is exchangeable, and moreover the projection in the first step is optional and can be omitted. If we set $g=0$ and remove the optional projection, the first step in \eqref{incre_prox_alg} becomes an unconstrained proximal map, and the second gradient step reduces to a projection step of its result. Upon inspection, this is equivalent to the update (\ref{imp_update_1st}) defining our stochastic proximal distance algorithm. Bridging these methodologies will allow us to leverage  techniques from that line of work toward convergence analysis in the following section. 

\paragraph{Contrast to Prior Work}  We have seen that Algorithm
1 is similar in spirit to implicit SGD. In particular, the inverse penalty $1/\rho_k$ plays the same role as the learning rate $\alpha_k$ in the SGD rule (\ref{implicit_sgd}). This observation is useful in the following section toward characterizing convergence and finite error bounds, incorporating properties of the constraint set $C$ into existing theoretical frameworks for gradient-based methods. The additional projection in \eqref{imp_update_1st} encourages iterates to stay near constraints throughout, in a sense performing a projected implicit step.

\citet{keys2019proximal} pointed out that the proximal distance algorithm reduces to the proximal gradient algorithm in the convex case. Consider the unconstrained problem
$$\min_{\bm{\theta}}\,\, F(\bm{\theta}) + q_\rho(\bm{\theta}) \; \st \; q_\rho(\bm{\theta}) = \frac{\rho}{2}\lVert \bm{\theta}-P_C(\bm{\theta}) \rVert^2:$$
since $\nabla q_\rho(\bm{\theta}) = \rho[\bm{\theta} - P_C(\bm{\theta})]$ in the convex case \citep{todd2003convex, lange2016mm}, the proximal distance algorithm for this problem with fixed $\rho$ can equivalently be written as:
\begin{equation*}
\bm{\theta}_{k} = \text{prox}_{\rho^{-1}F}[P_C(\bm{\theta}_{k-1})] 
= \text{prox}_{\rho^{-1}F}[\bm{\theta}_{k-1}-\rho^{-1}\nabla q_{\rho}(\bm{\theta}_{k-1})],
\end{equation*}
i.e. gradient descent on $q_\rho$ followed by a proximal operation on $F$ with step-size $\rho^{-1}$.  

There are two key differences in considering the \textit{stochastic versions} of the proximal distance and proximal gradient algorithms. First, with an increasing $\rho$ schedule, the component $q_{\rho}$---and thus objective sequence---is also changing, which is unusual as proximal gradient methods address a particular fixed objective. Second, while both algorithms modify their non-stochastic counterparts via subsampling, the randomness is introduced \textit{implicitly} in the stochastic proximal distance algorithm, and explicitly in the stochastic proximal gradient algorithm (even though the update scheme of the latter has an implicit nature). Partial intuition is conveyed by observing that the proximal mapping in the proximal distance method \textit{mirrors} the role it usually plays in related algorithms. That is, typically a gradient step is performed on the loss $F$, with a proximal map then applied to the (often non-smooth) regularizer $q$. Instead,  proximal distance updates \textit{reverse these roles}. In doing so, the stochasticity enters the iteration implicitly, which can translate to improved stability in practice. This in part leads to the desirable properties we prove below, and is further supported by our simulations in Section \ref{comp_IHT}. There we see notable improvements over projected SGD, a natural peer algorithm that can be viewed as an instance of stochastic proximal gradient descent for constrained problems.

In related work, \citet{asi2019stochastic} also considered the stochastic proximal point point algorithm for constrained problems. For problem \eqref{obj_intro}, their proposed iteration follows
\begin{equation}\label{asi2019algo}
\bm{\theta}_k = \argmin_{\bm{\theta} \in C} \big\{ f(\bm{\theta};\bm{z}_{\xi_k}) + \frac{1}{2\alpha_k} \lVert \bm{\theta}_k - \bm{\theta}_{k-1} \rVert^2\big\},
\end{equation}
where $\alpha_k$ is a diminishing learning rate, and they consider extensions so that $f(\cdot;\bm{z}_{\xi_k})$ in \eqref{asi2019algo} can be replaced by another convex function that approximates $f(\cdot;\bm{z}_{\xi_k})$ near $\bm{\theta}_{k-1}$. However, in this formulation the proximal update \eqref{asi2019algo} itself requires solving an inner constrained problem, and no implementation is included or considered. In contrast, the proximal update in our proposed algorithm \eqref{imp_update_1st} entails only an unconstrained problem, which is much easier computationally by design. For example, solving \eqref{asi2019algo} can be cumbersome even for an ordinary least squares loss function $f$, and can become more difficult depending on the specific constraint at hand. In comparison, our proposed method possesses a closed form solution \eqref{prox_update_lin} which is practical to implement whenever the projection is computable.

\section{Theoretical Analysis}
In this section, we establish convergence guarantees and finite error bounds for our proposed  algorithm in the convex case. Specifically, analysis is conducted on the projected sequence $P_C(\bm{\theta}_k)$. This is more straightforward and meaningful in non-asymptotic analysis, because $P_C(\bm{\theta}_k)$ satisfies the constraint while $\bm{\theta}_k$ does not. We present the assumptions and main results here with some discussion, with complete proof details for each provided in the Appendix.

In our analysis, we focus on polynomial schedules of increasing $\rho_k$. The first assumption guarantees that the Robbins-Monro conditions are satisfied within this class:

\begin{assumption}\label{rm_cond}
$\rho_k = \rho_1 \cdot k^\gamma$, $0.5< \gamma \le 1$.
\end{assumption}

One immediately sees that the ``learning rate" $\frac{1}{\rho_k}$ satisfies $\sum_{k=1}^{+\infty} \frac{1}{\rho_k} = +\infty$ and $\sum_{k=1}^{+\infty} \frac{1}{\rho_k^2} < +\infty$, which is common in proving convergence in stochastic optimization methods. 

Next, we make assumptions to place our theoretical analysis in the convex setting.
\begin{assumption}\label{conv_FC}
The constraint set $C$ is closed and convex, the objective $F$ is continuous, strictly convex and coercive.
\end{assumption}

In particular, this ensures the solution is well-posed:

\begin{proposition}\label{prop_unique}
Under Assumption \ref{conv_FC}, the constrained solution $\bm{\theta}_* = \argmin_{\bm{\theta}\in C}\,\,F(\bm{\theta})$ exists and is unique.
\end{proposition}

The next two assumptions enable us to analyze the stochastic proximal update in an implicit gradient manner:

\begin{assumption}\label{conv_com}
The loss components $f(\bm{\theta};\bm{z}_i)$ are differentiable and convex for any $\bm{z}_i$.
\end{assumption}

\begin{assumption}\label{lsmooth}
The gradient of the loss component $\nabla f(\bm{\theta}, \bm{z}_i)$ is $L$-Lipschitz for any $\bm{z}_i$.
\end{assumption}
This assumption is common in studies of implicit SGD \citep{toulis2015proximal, https://doi.org/10.48550/arxiv.2206.12663}, and holds for many common models such as generalized linear models (GLMs).

Now, we begin by establishing the following proposition, useful for the proof of the main theorems. It also indicates that the stochastic proximal distance algorithm is stable in some sense, because the two quantities below can both explode in SGD if the step size is too large.
\begin{proposition}\label{prop_bound}
Under Assumptions \ref{rm_cond}, \ref{conv_FC}, \ref{conv_com} and \ref{lsmooth}, denote $s := \sum_{k=1}^{+\infty} \frac{1}{k^{2\gamma}}$ and $G^2:= \mathbb{E}\{\lVert \nabla f(\bm{\theta}_*; \bm{z}_{\xi_k}) \rVert^2\} = \frac{1}{n} \sum_{i=1}^{n} \lVert \nabla f(\bm{\theta}_*; \bm{z}_i) \rVert^2$, then the iterative sequence and its gradient are bounded as follows.

(a). $\mathbb{E}\{ \lVert P_C(\bm{\theta}_k) - \bm{\theta}_* \rVert^2 \} \le \exp(\frac{4L^2 s}{\rho_1^2}) \cdot \{\frac{4G^2 s}{\rho_1^2}  + \lVert P_C(\bm{\theta}_0) - \bm{\theta}_* \rVert^2\} := r^2$

(b). $\mathbb{E}\{ \lVert \nabla f[P_C(\bm{\theta}_{k-1}); \bm{z}_{\xi_k}] \rVert^2 \} \le 2(G^2 + r^2 L^2) := c^2$
\end{proposition}

Our first theorem establishes convergence to the constrained solution with an appropriate polynomial rate schedule for $\rho_k$. The result is analogous to Proposition 9 of \citet{bertsekas2011incremental}, but our proof does \textit{not} assume bounded gradients. 

\begin{theorem}\label{prop_conv}
Under Assumptions \ref{rm_cond}, \ref{conv_FC}, \ref{conv_com} and \ref{lsmooth}, $P_C(\bm{\theta}_k)$ converges to $\bm{\theta}_* = \argmin_{\bm{\theta}\in C}\,F(\bm{\theta})$ almost surely.
\end{theorem}

Next, we impose an ``identifiability condition" in that there is enough curvature near the global solution so that it is sufficiently separated from other potential parameter values.

\begin{assumption}\label{strong_conv}
For all $\bm{\theta} \in C$, the loss function $F$ satisfies
$$\nabla F(\bm{\theta})^T(\bm{\theta} - \bm{\theta}_*) \ge  F(\bm{\theta}) - F(\bm{\theta}_*)  + \frac{\mu}{2}\lVert \bm{\theta} - \bm{\theta}_* \rVert^2.$$ 
\end{assumption}

Note that Assumption \ref{strong_conv} is implied by strong convexity, which is typical in finite error bound analysis in similar analyses of related methods \citep{10.2307/26362880}. That is, our condition is strictly weaker than those imposed in such studies---we only require the inequality to hold at $\bm{\theta}_*$ rather than everywhere on the domain. Next, we characterize the convergence behavior of the iterate sequence in various regimes of the rate parameters. 

\begin{theorem}\label{finite_para}
Under Assumptions \ref{rm_cond}, \ref{conv_FC}, \ref{conv_com}, \ref{lsmooth} and \ref{strong_conv}, denote $\delta_k := \mathbb{E}\big[ \lVert P_C(\bm{\theta}_k) - \bm{\theta}_* \rVert^2 \big]$, and let $c, r, G, s$ as defined in Proposition \ref{prop_bound}. Then it holds that 
$$
\delta_k \le \frac{2m(1 + \frac{\mu}{\rho_1})}{\mu \rho_1}k^{-\gamma}  + (1+\frac{\mu}{\rho_1})^{-\phi_\gamma (k) }(\delta_0 + A)$$
where $m = 5c^2 + 2L^2 r^2 + 2LGr$, \, $A = \frac{ms}{\rho_1^2} (1 + \frac{\mu}{\rho_1})^{k_0}$, \, $k_0 > 0$ is a constant, \, and
\[ \phi_\gamma (k) = \begin{cases}
  k ^{1-\gamma} & \text{if } \quad \gamma \in (0.5, 1) \\
  \log k & \text{if } \quad  \gamma = 1.
\end{cases} \] 
\end{theorem}

In particular, this indicates that the convergence rate for $\lVert P_C(\bm{\theta}_k) - \bm{\theta}_* \rVert^2$ is $\mathcal{O}(k^{-\gamma})$ whenever $\gamma \in (0.5,1)$. Otherwise, if $\gamma=1$, it becomes $\mathcal{O}(k^{-\min\{1, log(1
+\frac{\mu}{\rho_1})\}})$.  
Another interesting takeaway is that the forgetting rate of the initial error $\delta_0$ is $(1+\frac{\mu}{\rho_1})^{-\phi_\gamma (k)}$, \textit{regardless of the initial penalty} $\rho_1$. This is one rigorous characterization of the stability properties inherited by our stochastic proximal distance algorithm as an implicit scheme. In contrast, for explicit stochastic gradient algorithms such as SGD, it is well known that the initial error can be amplified arbitrarily if the initial learning rate is improper \citep{moulines2011non,konevcny2015mini}. We next establish a similar result characterizing  global convergence of the objective sequence.

\begin{theorem} \label{finite_F}
Let Assumptions \ref{rm_cond}, \ref{conv_FC}, \ref{conv_com}, \ref{lsmooth} and \ref{strong_conv}, hold, and denote $\zeta_k := \mathbb{E}\{F[P_C(\bm{\theta}_k)] - F(\bm{\theta}_*)\}$. Then,
\begin{equation*}
\begin{split}
\zeta_k &\le 
G\sqrt{\frac{2m(1 + \frac{\mu}{\rho_1})}{\mu \rho_1}}k^{-\frac{\gamma}{2}} 
+ G (1+\frac{\mu}{\rho_1})^{-\frac{\phi_\gamma (k)}{2} } \sqrt{\delta_0 + A} \\
&\quad + B k^{-\gamma}  + L\cdot(1+\frac{\mu}{\rho_1})^{-\phi_\gamma (k) }(\delta_0 + A)
\end{split}
\end{equation*}
where $A$ and $\phi_\gamma$ are as defined in Theorem \ref{finite_para}, \, $\delta_0 = \lVert P_C(\bm{\theta}_0) - \bm{\theta}_* \rVert^2$, \, and $\displaystyle B = \frac{\mu c^2 + 2Lm(1 + \frac{\mu}{\rho_1})}{\mu \rho_1}$.
\end{theorem}

This result indicates that the convergence rate of the objectives is 
$\mathcal{O}(k^{-\frac{\gamma}{2}})$ if $\gamma \in (0.5, 1)$ and becomes $\mathcal{O}(k^{-\min[\frac{1}{2}, \frac{1}{2}log\,(1+\frac{\mu}{\rho_1})]})$ for $\gamma = 1$. It follows that the fastest worst-case rate it can achieve is $\mathcal{O}(k^{-\frac{1}{2}})$. We summarize the convergence rate regimes for both the parameter and the objective sequence in Table \ref{cov_rates}. Moreover, we see that the convergence rate is consistent with that of the sequence of iterates $\mathbb{E}\lVert P_C(\bm{\theta}_k) - \bm{\theta}_* \rVert^2$ in the sense that both rates are monotonically increasing with respect to $\gamma$. This differs from a similar analysis of the idealized proximal Robbins-Monro method in \citep{toulis2015proximal}, where the authors show that the convergence rate of $\mathbb{E}\lVert \bm{\theta}_k - \bm{\theta}_* \rVert^2$ is also $\mathcal{O}(k^{-\gamma})$, but is at odds with the convergence rate of $\mathbb{E}[F(\bm{\theta}_k) - F(\bm{\theta}_*)]$ which is non-monotonic in $\gamma$. Indeed, the best rate they obtain is of order $\mathcal{O}(k^{-\frac{1}{3}})$, and is achieved at $\gamma = \frac{2}{3}$ in their analysis.

\begin{table}[htbp]
    \centering
    \caption{Convergence rates}
    \label{cov_rates}
    \begin{tabular}{c|cl}
\hline
& $\gamma \in (\frac{1}{2}, 1)$ & $\gamma = 1$ \\ \hline
$\lVert P_C(\bm{\theta}_k) - \bm{\theta}_* \rVert^2$ &    $\mathcal{O}(k^{-\gamma})$                           &     $\mathcal{O}(k^{-\min\{1, log(1
+\frac{\mu}{\rho_1})\}})$         \\
$F[P_C(\bm{\theta}_k)] - F(\bm{\theta}_*)$           &         $\mathcal{O}(k^{-\frac{\gamma}{2}})$                      &            $\mathcal{O}(k^{-\min[\frac{1}{2}, \frac{1}{2}log\,(1+\frac{\mu}{\rho_1})]})$ \\
\hline
\end{tabular}
\end{table}

We briefly remark that all proofs of the above results apply to the mini-batch (and full batch) case with only slight modifications. For the mini batch with batch size $b$ case, A subset $\mathcal{I}_k \subset \{1,...,n\}$ of size $b$ is drawn without replacement at each iteration, and $f(\bm{\theta};\bm{z}_{\xi_k})$ is replaced by $\frac{1}{b}\sum_{i\in \mathcal{I}_k} f(\bm{\theta};\bm{z}_i)$; the results hold immediately since the latter satisfies the same conditions as $f(\bm{\theta};\bm{z}_{\xi_k})$. 

\paragraph{Novelty of Theoretical Devices} 
It is worth emphasizing several ways our analyses improve upon prior techniques. First, to establish almost sure convergence in Theorem \ref{prop_conv}, Assumption \ref{lsmooth} imposes only that the \textit{gradient} of the objective is Lipschitz. Our condition is significantly less restrictive than prior work guaranteeing almost sure convergence, and importantly can be verified on many common model classes such as GLMs. In contrast, the technique by \citet{bertsekas2011incremental} requires that the gradient of each component function has bounded norm almost surely. From our strictly weaker condition, we may derive Proposition \ref{prop_bound} (b), which substitutes for their quite strong condition a similar role in the proof. Importantly, a merit of implicit schemes lies in their advantage to avoid the instability from potentially explosive behavior in explicit gradient schemes. Assuming bounded gradient components as in prior work precludes this possibility, missing this ``interesting" gap by restricting to  a regime where explicit schemes also succeed.

Next, our finite error bounds in Theorem \ref{finite_para} do not require that the objective is twice differentiable and Lipschitz as in previous studies of implicit SGD \citep{10.2307/26362880}. More importantly, these past analyses posit that the loss is globally both Lipschitz and strongly convex, which  cannot simultaneously hold true \citep{asi2019stochastic,https://doi.org/10.48550/arxiv.2206.12663}. 
While we are inspired by their analyses, our proof techniques for analyzing the finite error bounds of the objective function make a departure from recent studies of  proximal algorithms by \citet{toulis2015proximal,https://doi.org/10.48550/arxiv.2206.12663}. They derive a recursion relating the objective error at $\bm{\theta}_{k-1}$ and $\bm{\theta}_k$, and employ a clever technique to analyze this recursion through constructing another sequence for bounding the error, resulting in a case-based analysis depending on the value of $\gamma$. In our setting, it is not obvious how to leverage these recursive arguments, as the projection operator complicates analysis. Instead, we study the relation between $\zeta_k$ and $\delta_k$, an approach that does not require discussing different cases of $\gamma$ values. As a result, our error bounds not only yield superior rates of convergence, but are arguably more natural in this sense.

\section{Empirical Study}
\subsection{Validation of Theoretical Results}
As the breadth of proximal distance algorithms has been demonstrated in past works \citep{NIPS2017_061412e4,landeros2022extensions}, we focus on assessing whether the theoretical results align with performance in practice on synthetic experiments. 

First we consider linear and logistic regression in low and high dimensional settings. Specifically, we consider a simple convex example involving the unit ball constraint  $C_u = \{\bm{\theta} \in \mathbb{R}^p: \lVert \bm{\theta} \rVert_2 \le 1\}$, as well as the non-convex exact sparsity set  $C_s=\{\bm{\theta} \in \mathbb{R}^p: \lVert \bm{\theta} \rVert_0 \le s\}$, where $p$ is the dimension of covariates and the sparsity level we choose is $s=5$. 
In both cases, projection is very efficient: $P_{C_u}(\bm{\theta})$ is obtained by simply scaling $\bm{\theta}$ if its norm is larger than $1$, and $P_{C_s}(\bm{\theta})$ is computed by setting all but the largest (in absolute value) $s$ elements to $0$. We will find that while our guarantees are sufficient for convergence, the algorithm often succeeds when not all conditions are met, exploring behavior for $\gamma$ outside of the guaranteed range and considering \textit{nonconvex} constraint such as $C_s$. We describe each simulation study below, with complete details on data generation and deriving the proximal updates in the Supplement.

\paragraph{Constrained Linear Regression} 
At each iteration of Algorithm \ref{stoalgo}, denote the subsampled data as $\tilde{\bm{X}}\in \mathbb{R}^{b\times p}$ and corresponding observations $\tilde{\bm{y}}\in \mathbb{R}^b$, where $b$ is the batch size. In the linear setting, the proximal update for $\bm{\theta}$ has closed form:
\begin{equation}\label{prox_update_lin}
\begin{split}
\bm{\theta}_k &= (b\rho_k \bm{I}_p + \tilde{\bm{X}}^T\tilde{\bm{X}})^{-1}[b\rho_k P_C(\bm{\theta}_{k-1}) + \tilde{\bm{X}}^T\tilde{\bm{y}}] \\
&= [\bm{I}_p - \tilde{\bm{X}}^T(b\rho_k \bm{I}_b + \tilde{\bm{X}}\tilde{\bm{X}}^T)^{-1}\tilde{\bm{X}}][P_C(\bm{\theta}_{k-1}) + 
b^{-1}\rho_k^{-1}\tilde{\bm{X}}^T\tilde{\bm{y}}] .
\end{split}
\end{equation}
The second equation follows by invoking the Woodbury formula, and significantly reduces computations when $b\ll p$. We consider various combinations of $(n,p)$, and keep the batch size to be $5\%$ of the sample size.

\paragraph{Constrained Logistic Regression} 
Unlike the linear case above, there is no closed form solution for the proximal operation in constrained logistic regression. We employ a Newton iteration with update rule as follows: 
\begin{equation*}
\begin{split}
&\text{For }i = 1,...,S \text{ Newton steps} \\
&\qquad \bm{\beta}_i = \bm{\beta}_{i-1} - \eta \cdot [\rho_k \bm{I}_p + \frac{1}{b}\tilde{\bm{X}}^T\tilde{\bm{W}}\tilde{\bm{X}}]^{-1} 
\{-\frac{1}{b} \tilde{\bm{X}}^T(\tilde{\bm{y}} - \tilde{\bm{p}}) + \rho_k[\bm{\beta}_{i-1} - P_C(\bm{\theta}_{k-1})]\} .
\end{split}
\end{equation*}
Here $\tilde{\bm{p}}$ is a $b\times 1$ vector, $\tilde{\bm{W}}$ is a $b\times b$ diagonal matrix; both depend on the parameter $\bm{\beta}_{i-1}$. The stepsize $\eta$ is selected by Armijo backtracking, and the inverse matrix above again can be rewritten to reduce complexity when $b\ll p$:

\begin{equation*}\rho_k^{-1} \bm{I}_p - \rho_k^{-1}\tilde{\bm{X}}^T(b\rho_k \bm{I}_b + \tilde{\bm{W}}\tilde{\bm{X}}\tilde{\bm{X}}^T)^{-1}\tilde{\bm{W}}\tilde{\bm{X}}. 
\end{equation*}

We take subsamples of size $20\%$ of the dataset.

\begin{figure*}[t]
    \centering
    \includegraphics[width=.97\textwidth]{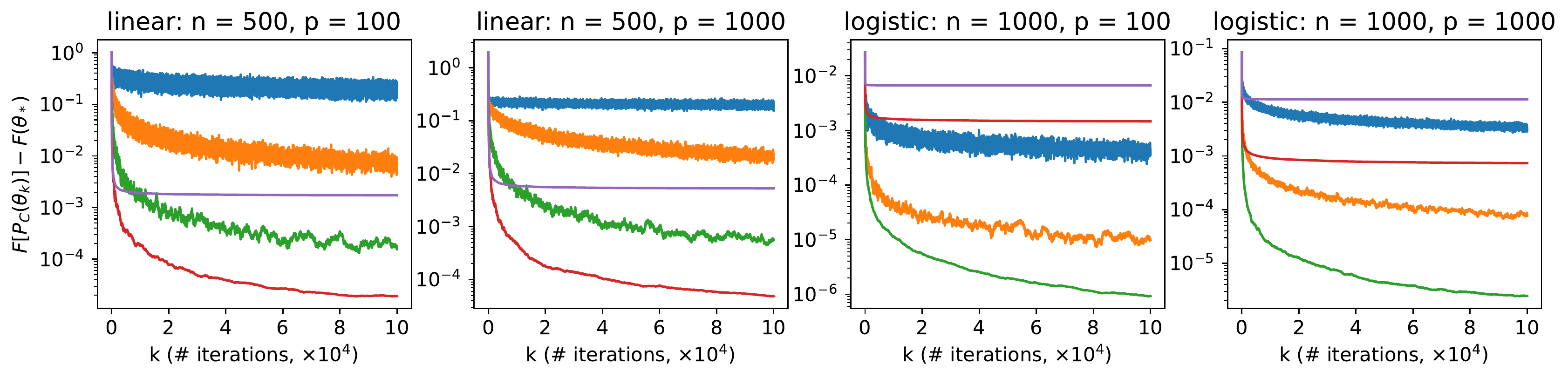}
    \includegraphics[width=.97\textwidth]{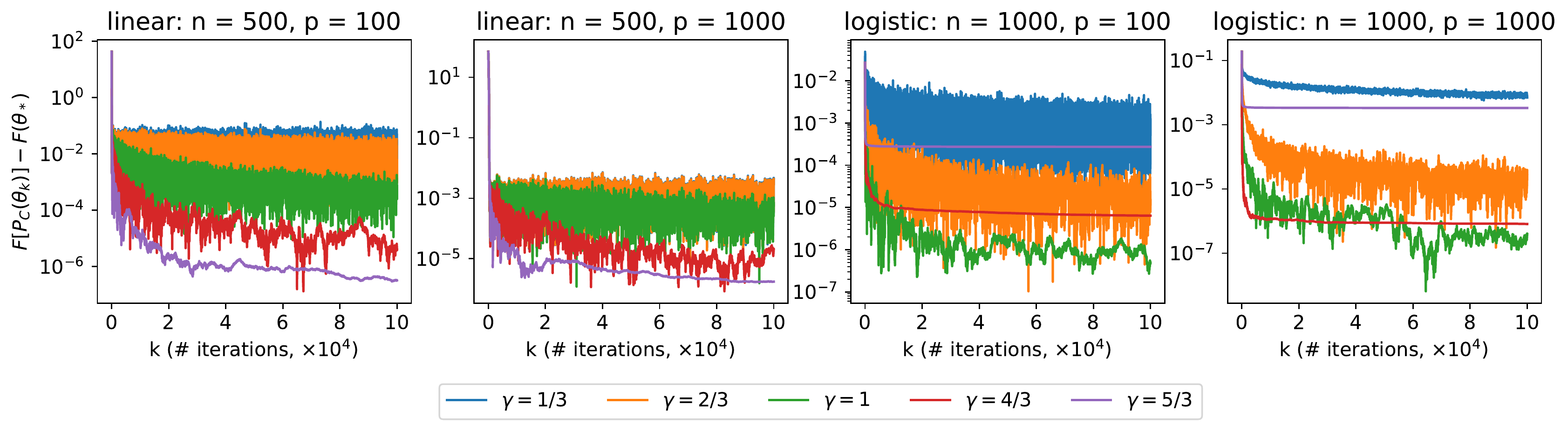}
    \caption{Finite error at different $\gamma$ in different settings under the unit ball constraint (top panel) and sparsity set constraint (bottom panel)}
    \label{conv_path} 
\end{figure*}

\begin{figure*}[!h]
    \centering
    \includegraphics[width=1\textwidth]{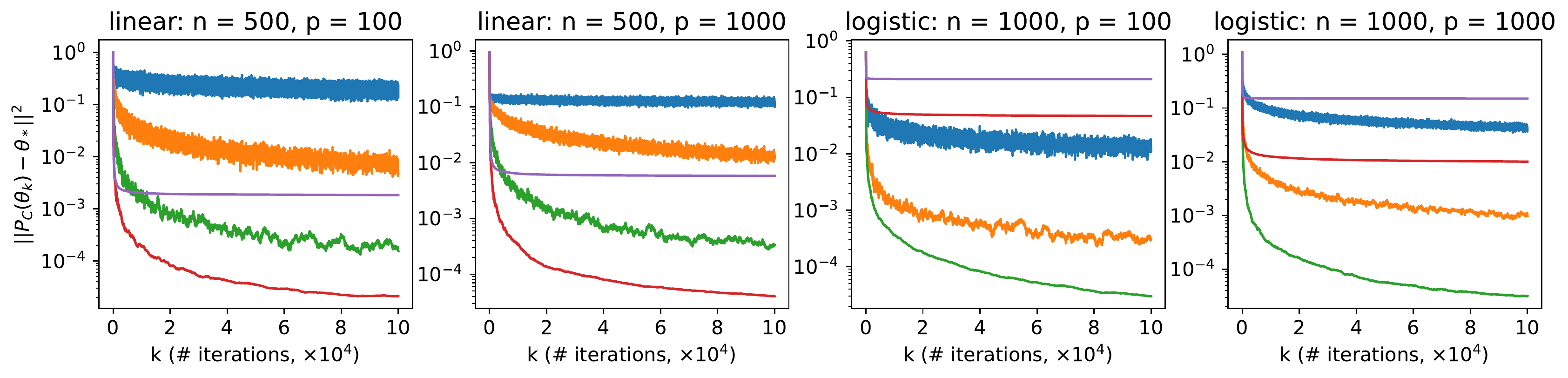}
    \centering
    \includegraphics[width=1\textwidth]{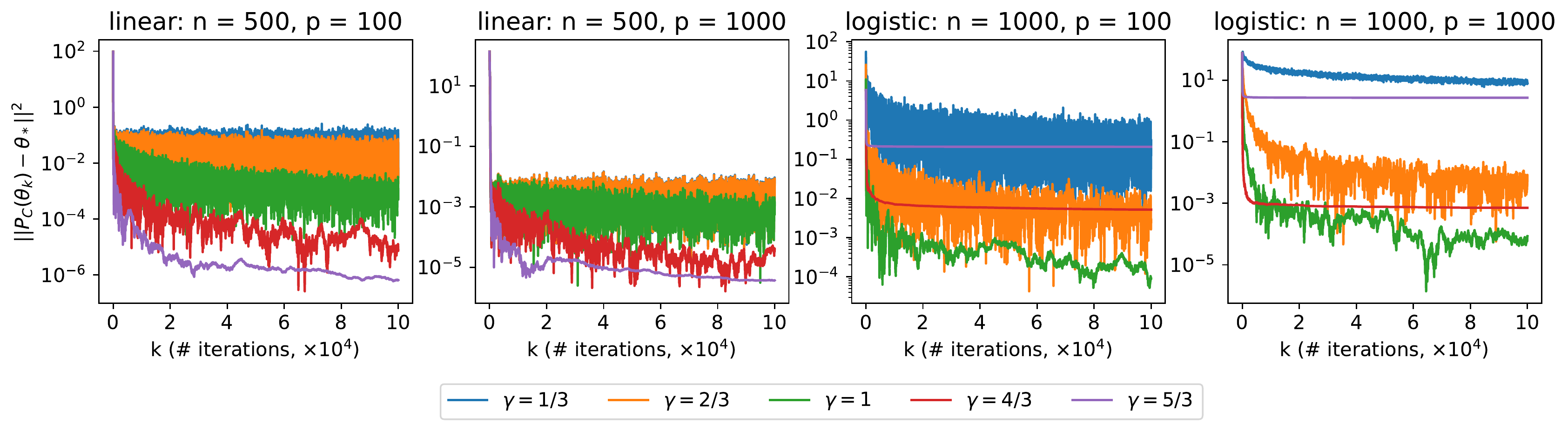}
    \caption{Finite error in terms of the parameter at different $\gamma$ in different settings under the unit ball constraint (top panel) and sparsity constraint (bottom panel)}
    \label{conv_path_theta}
\end{figure*}

\paragraph{Simulation Results}
We begin by validating the finite error properties in Theorems \ref{finite_para} and \ref{finite_F} under various penalty rates $\gamma$ from a fixed initial penalty. Though our proofs assume $\gamma$ lies in $(\frac{1}{2}, 1]$, we explore values beyond this range as well in $\{\frac{1}{3}, \frac{2}{3}, 1, \frac{4}{3}, \frac{5}{3}\}$ since Assumption \ref{rm_cond} is sufficient but not necessary. We set the same seed across different $\gamma$ settings to control the stochasticity due to subsampling in our comparison; we set $\rho_1=10^{-1}$ in the unit ball example and $\rho_1=10^{-3}$ in the sparsity constraint case.

Figure \ref{conv_path} summarizes the results in terms of objective convergence $F[P_C(\bm{\theta}_k)] - F(\bm{\theta}_*)$. The analogous plots for the optimization variable $\lVert P_C(\bm{\theta}_k) - \bm{\theta}_* \rVert^2$ are displayed in Figure \ref{conv_path_theta} and convey the same trends; this is to be expected as Theorems \ref{finite_para} and \ref{finite_F} yield the same rate regimes for the iterate sequence and objective sequence. Moreover, in all settings, the algorithms succeed in making steady progress under values in the guaranteed range (albeit slowly for $\gamma=1/3$). In line with our convergence theory, the ranking of their rates of progress is monotonic in  $\gamma\in\{\frac{1}{3}, \frac{2}{3}, 1\}$. Interestingly, this trend continues beyond $\gamma>1$ in some settings. For example, in the linear case under unit ball constraint, as $\gamma$ increases to $\frac{4}{3}$, the algorithm continues to make even faster progress. However, the method appears to approach an asymptote prematurely for $\gamma = \frac{5}{3}$, a pattern that emerges starting from $\gamma = \frac{4}{3}$ in the logistic case. Similar patterns are observed for both models under the non-convex sparsity constraint. This suggests that the sufficient conditions in our general analysis are not necessary ones. In particular, it may be possible to derive sharper characterizations of convergence rates when additional ``nice" structure of the objectives hold on a case by case basis, such as in our linear examples. We also note that the paths in the sparse regression example appear choppier than those subject to the unit ball, which may be due to the non-convexity of the sparsity set.

These trends agree with the results suggested by our theoretical analysis. Noting that 
increasing $\rho$ is advantageous for the first term to decrease but disadvantageous for the second term in the bound given by Theorem \ref{finite_para}, we additionally provide a thorough experimental study of the behavior and interaction between parameters. These additional empirical results appear in the Supplement to respect space limitations, and convey a more complete understanding of the qualitative behavior of the algorithm as $\rho_1,\gamma$ vary.

\subsection{Extended Comparisons}\label{comp_IHT}
We now extend our empirical study, shifting the focus away from validating our theoretical findings to instead compare to the popular peer method of projected SGD. This study also showcases that our method is readily amenable to settings including matrix-variate optimization and under Huber loss, beyond the more standard setups in the previous section. Comprehensive details on data generation and derivation of proximal mappings under these models are provided in the Supplement.

In the case of sparsity constraints, projected SGD is equivalent to stochastic gradient hard thresholding \citep{jain2014iterative, 8025727}; The algorithm wraps stochastic gradient descent within a projection according to
\begin{equation}\label{IHT}
\theta_k = P_C\big[\theta_{k-1} - \alpha_k \nabla f(\theta_{k-1}; z_{\xi_k})\big],
\end{equation}
where $\alpha_k$ is a diminishing learning rate. Such a comparison is lacking in prior work \citep{NIPS2017_061412e4} which focuses on competitors under shrinkage penalties, but it is natural to consider this peer method. Note if we rewrite \eqref{obj_intro} as $\min_\theta\; F(\theta) + q_C(\theta)$, where $q_C(\theta) = 0$ if $\theta \in C$ and $q_C(\theta) = +\infty$ if $\theta \notin C$, we can see that \eqref{IHT} reduces to a stochastic proximal gradient algorithm that we have discussed at the end of Section \ref{our_method}, since $\text{Prox}_{\alpha_k q_C}(\cdot) = P_C(\cdot)$.

\paragraph{Constrained Huber Regression}
Huber regression is a technique used to improve the robustness of linear regression models against outliers. Instead of the squared loss in an ordinary linear model, it applies the Huber loss 
$$L_\delta(a) = \begin{cases}
  \frac{1}{2}a^2 & \text{if } \quad \lvert a \rvert \le \delta\\
  \delta (\lvert a \rvert - \frac{1}{2}\delta) & \text{if } \quad \lvert a \rvert > \delta
\end{cases}$$
to the regression residuals, i.e. using $f(\bm{\theta}; \bm{x}_i, y_i) = L_\delta(y_i - \bm{x}_i^T \bm{\theta})$ in \eqref{obj_intro}. 
As before, we consider both the unit ball constraint $C_u$ and the exact sparsity constraint $C_s$ in our study. In addition, we now randomly let $10\%$ of the observations to be outliers when generating the data, and we set $\delta = 2$ for the Huber loss function.

At each iteration, denote the subsampled data as $\tilde{\bm{X}}\in \mathbb{R}^{b\times p}$ and $\tilde{\bm{y}}\in \mathbb{R}^b$, where $b$ is the batch size. There is no closed form solution for the proximal mapping with Huber loss, but the loss function is convex, and can be computed with a few iterations of gradient descent with update rule 
\begin{equation*}
\begin{split}
& \, \; \beta_j = \beta_{j-1} - \eta \{-\frac{1}{b}\tilde{X}^T \tilde{l} + \rho_k[\beta_{j-1} - P_C(\theta_{k-1})] \}
\end{split}
\end{equation*}
where $\tilde{\bm{l}}$ is a $b \times 1$ vector that depends on $\beta_{j-1}$, and $\eta$ is a step size.

\paragraph{Constrained Matrix Regression}
Matrix regression models regress a response variable on matrix structured covariates, and the regression parameter is also structured in a matrix \citep{10.1111/rssb.12031}. The loss function \eqref{obj_intro} in this case becomes $f(\bm{\Theta}; \bm{X}_i, y_i) = (y_i - \langle \bm{X}_i \,,\, \bm{\Theta} \rangle)^2$, where $y_i \in \mathbb{R},\, \bm{X}_i,\, \bm{\Theta} \in \mathbb{R}^{p\times q}$, and $\langle \bm{X}_i \,,\, \bm{\Theta} \rangle = \text{trace}(\bm{X}_i^T \bm{\Theta}) = \text{vec}(\bm{X}_i)^T\text{vec}(\bm{\Theta})$ \footnote{$\text{vec}(\cdot)$ stacks the columns of a matrix into a vector}.

In place of imposing sparsity  as in many vector-valued regression problems, low-rank structure is often a more appropriate structural assumption for encouraging parsimony in the true signal for matrix-variate data \citep{10.1111/rssb.12031}. Following \citet{NIPS2017_061412e4}, we consider matrix regression under the exact low-rank constraint $C_r = \{\bm{\Theta} \in \mathbb{R}^{p \times q}: \text{rank}(\bm{\Theta}) \le r\}$. 

For implementation of our proposed Algorithm \ref{stoalgo}, the projection $P_{C_r}(\bm{\Theta})$ is obtained by performing SVD on $\bm{\Theta}$ and then reconstructing using only the $r$ largest singular values and corresponding singular vectors, following the Eckart-Young theorem. The proximal update for $\bm{\Theta}$ has a closed form, and its vectorized expression can be obtained using \eqref{prox_update_lin}.

\paragraph{Simulation Results}
In all settings, we generate a dataset comprised of $n=10\;000$ observations, which is a suitable scenario for applying stochastic algorithms. We set the number of covariates as $p=1\;000$ for linear, Huber, logistic regression models, and dimensions $p = q = 64$ for the matrix regression model. We take a batch size $b = 50$ for linear, Huber and matrix regression models, with $b = 200$ for the logistic model. We match linear learning rates between our method and projected SGD, i.e. $\rho_k = k \rho_1$ and $\alpha_k = \alpha_1 / k$, and tune their initial learning rates.

\begin{table}[htbp]
    \centering
    \caption{Comparison between the stochastic proximal distance algorithm (SPD) and projected SGD algorithm (PSGD). In sparsity settings, the true discovery rate is denoted in  parentheses if it is not $1$. For references, $\lVert \bm{\theta}_* \rVert^2 \approx 155$ and $620$ in ``sparsity = 5'' settings and ``sparsity = 20'' settings, and $\lVert \bm{\Theta}_* \rVert_F^2 = 128$ in all settings}
    \label{compare_results}
    \begin{tabular}{ll|cc}
    \toprule
    \multicolumn{2}{c|}{} & \multicolumn{2}{c}{$\lVert \hat{\bm{\theta}} - \bm{\theta}_* \rVert^2$ (or $\lVert \hat{\bm{\Theta}} - \bm{\Theta}_* \rVert_F^2$)} \\
    \cmidrule(lr){3-4}
    \multicolumn{1}{c}{Model} & \multicolumn{1}{c|}{Constraint} & \multicolumn{1}{c}{SPD} & \multicolumn{1}{c}{PSGD} \\
    \midrule
    \multirow{3}{*}{Linear} & sparsity = 5 & 0.002 & 0.005 \\
    & sparsity = 20 & 0.006 & 0.006 \\
    & unit-ball & 0.030 & 0.038 \\
    \midrule
    \multirow{3}{*}{Logistic} & sparsity = 5 & 1.509 & 3.087 (0.98) \\
    & sparsity = 20 & 23.91 & 211.0 (0.86) \\
    & unit-ball & 0.165 & 0.167 \\
    \midrule
    \multirow{3}{*}{Huber} & sparsity = 5 & 0.005 & 0.003 \\
    & sparsity = 20 & 0.029 & 0.018 \\
    & unit-ball & 0.051 & 0.054 \\
    \midrule
    \multirow{3}{*}{Matrix} & rank = 1 & 0.015 & 0.011 \\
    & rank = 2 & 0.021 & 5.001 \\
    & rank = 5 & 0.041 & 100.1 \\
    \bottomrule
    \end{tabular}
\end{table}

The averaged results over $50$ random repeats for each experiment are presented in Table \ref{compare_results}. It records the squared error and the true discovery rate for sparsity settings. From Table \ref{compare_results}, we can see that although the projected SGD algorithm is comparable with our proposed method in all the unit ball constraint settings, its performance under the more challenging sparsity and rank constraints is unsatisfactory. Under sparsity or rank restriction, the performance of the two peer methods is somewhat comparable in linear and Huber regression models, but projected SGD performs much worse than the proposed method in the logistic and matrix regression models. We can also see that the performance of projected SGD worsens as the sparsity/rank increases. It is likely that especially in the non-convex settings, truncating gradient information can de-stabilize the performance of projected SGD, while our method not only incorporates the constraint information via projection in an unconstrained way, but is based on an implicit gradient descent scheme derived as a variation of an MM algorithm. 

\paragraph{Computational Complexity}\begin{table}[htbp]
\centering
\captionsetup{width = 0.45\textwidth}
\caption{Complexity of projections}
\label{complexity_proj}
\begin{tabular}{c|c}
\hline
Constraint & Projection cost                   \\ \hline
unit-ball  & $\mathcal{O}(p)$              \\
sparsity   & $\mathcal{O}(p\log p)$        \\
rank       & $\mathcal{O}(pq\min\{p, q\})$ \\ \hline
\end{tabular}
\end{table}
This section overviews the computational complexity of our proposed method. The complexity varies case by case as the projection operator---often the bottleneck---depends on the specific constraint set, and the proximal mapping depends on the loss function. We summarize the computational complexity of the proposed method next to that of projected SGD in all of the simulated examples in this article. The worst case time complexity for the projection operations and for the proximal and gradient updates are shown in Table \ref{complexity_proj} and Table \ref{complexity_update}, respectively. Recall $p$ denotes the dimension of $\bm{\theta}$, $(p, q)$ is the dimension of $\bm{\Theta}$ in matrix regression, $b<p$ is the batch size, and $S$ is the number of iterations to solve the proximal mapping in logistic and Huber regression problems. For both our proposed method and the projected SGD method, the overall computational complexity is the bottleneck of the projection and the proximal/gradient update rule.

\subsection{Real Data Case Study}

\begin{table}[t]
\centering
\captionsetup{width = 0.45\textwidth}
\caption{Computational complexity of proximal mapping and gradient step}
\label{complexity_update}
\begin{tabular}{ccc}
\hline
Model                     & Proximal Mapping                      & Gradient Descent                     \\ \hline
\multirow{2}{*}{Linear}   & \multirow{2}{*}{$\mathcal{O}(p^2b)$}  & \multirow{2}{*}{$\mathcal{O}(p^2b)$} \\
                          &                                       &                                      \\ \hline
\multirow{2}{*}{Logistic} & \multirow{2}{*}{$\mathcal{O}(p^2bS)$} & \multirow{2}{*}{$\mathcal{O}(pb)$}   \\
                          &                                       &                                      \\ \hline
\multirow{2}{*}{Huber}    & \multirow{2}{*}{$\mathcal{O}(pbS)$}   & \multirow{2}{*}{$\mathcal{O}(pb)$}   \\
                          &                                       &                                      \\ \hline
Matrix                    & $\mathcal{O}(p^2q^2b)$                & $\mathcal{O}(p^2q^2b)$               \\ \hline
\end{tabular}
\end{table}

We apply our method to classification in an unbalanced real dataset on oral toxicity \footnote{Publicly available at \url{https://archive.ics.uci.edu/ml/datasets/QSAR+oral+toxicity}} containing $1024$ features (molecular fingerprints) and one binary response variable (toxic or not toxic), measured over $8992$ observations (chemicals). This case study serves to illustrate performance in real data, as well as a comparison to stochastic projected gradient descent \citep{jain2014iterative, 8025727}. 

We randomly split $80\%$ and $20\%$ of the data as the training set and testing set. For our method, we choose $\rho_k = \rho_1 \cdot k$, and  use $10$-fold cross validation to tune the sparsity parameter $s$ and initial penalty $\rho_1$, yielding the choices  $s = 100$ and $\rho_1 = 1\times 10^{-2}$. For the projected stochastic gradient descent method, we use the same sparsity level $s = 100$ and tune its initial step size analogously. We run each competing method across $30$ random replicates under batch sizes $200$ and $500$. The result shows that our proposed stochastic proximal distance algorithm scales to the data well, and achieves favorable prediction performance. While slower than stochastic projected gradient descent, the proposed method converges in under a minute in all cases. 
More importantly, the predictive accuracy of the stochastic proximal distance algorithm--- especially visible via AUC ---is significantly higher than stochastic PGD. This is  evident in Figure \ref{comb_a}, with complete details of experimental settings included in the Supplement.

\begin{figure}[!h]
    \centering
  \begin{subfigure}[b]{0.49\textwidth}
  \centering
  \includegraphics[width=1\textwidth]{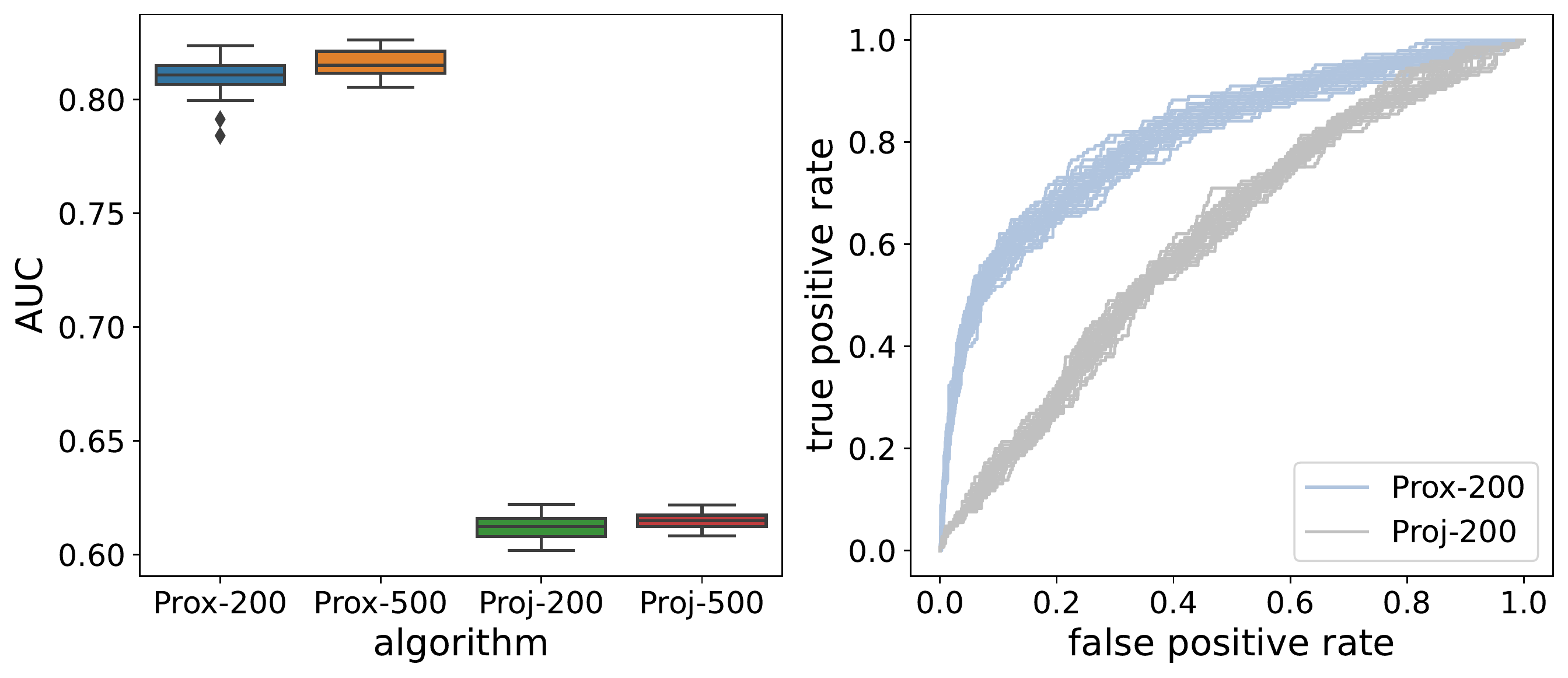}
  \caption{Comparison in terms of AUC and ROC, oral toxicity data.}
  \label{comb_a}
  \end{subfigure}
  \hfill
  \begin{subfigure}[b]{0.49\textwidth}
  \centering
  \includegraphics[width=1\textwidth]{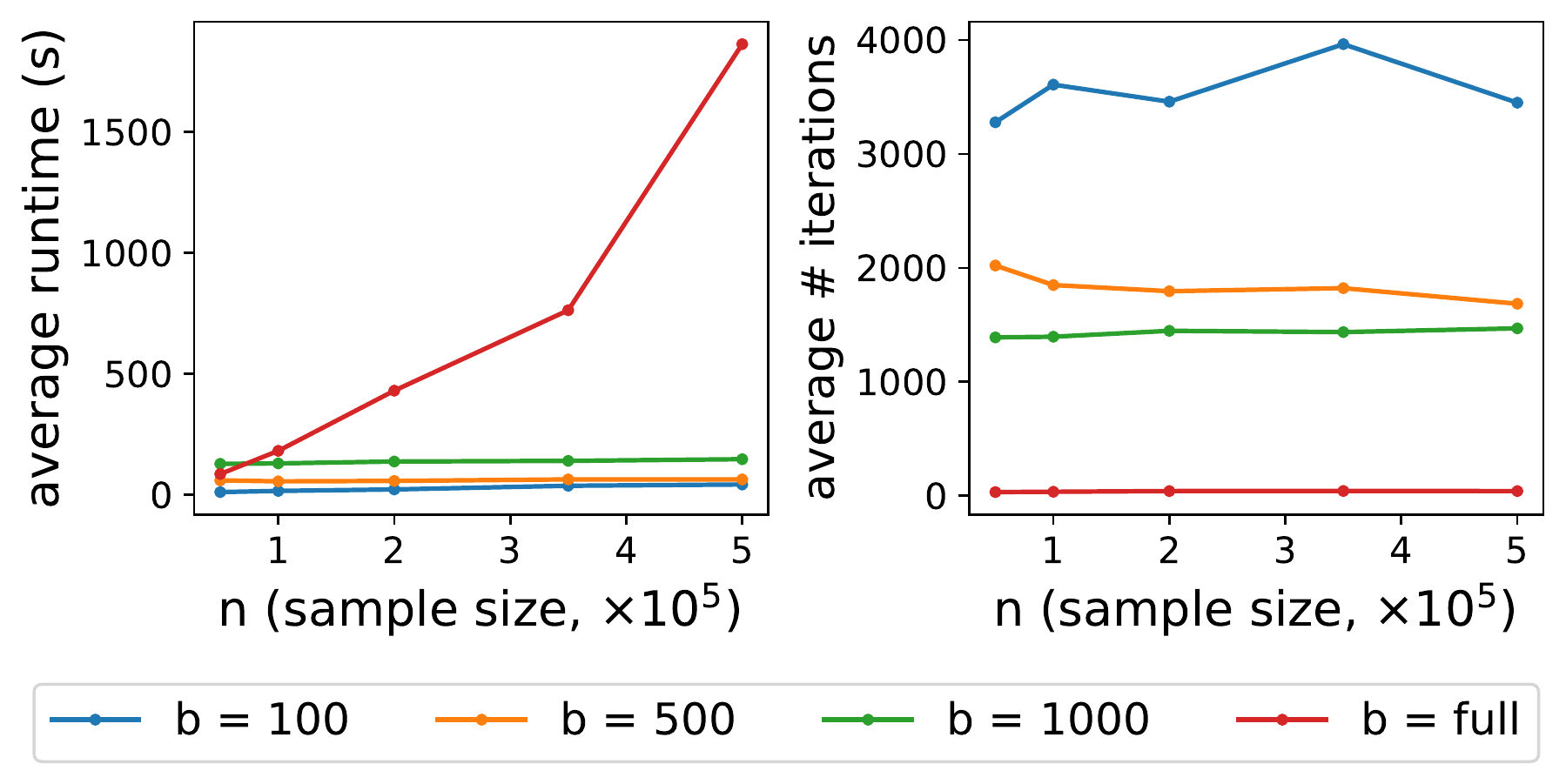}
  \caption{Runtime as sample size and batch size vary.}
  \label{comb_b}
  \end{subfigure}
  \caption{(a) Performance on real data. "Prox" refers to the proposed method; "Proj" is projected stochastic gradient descent, with batch sizes $b=200$ and  $500$; (b) Runtime as $n,b$ vary}
  \label{combined_pic}
\end{figure}

\subsection{Computational Advantages}
Finally, we present a runtime comparison to show that unsurprisingly, the stochastic version of the proximal distance algorithm outpaces its full batch counterpart while maintaining comparable estimation error.  We run a focused study on logistic regression in dimension $p=1000$,  varying the sample size from $n= 5 \times 10^4$ to $5 \times 10^5$. Performance is assessed under sampling batches of size $b = 100,500,1000,n$. Notably, these comparisons are easily carried out on a single laptop machine. 
Results in terms of wall-clock time as well as iterations until convergence are displayed in Figure \ref{comb_b}. Though the stochastic algorithm requires more iterations until convergence, a simple back-of-the-envelope calculation by multiplying the batch size and iteration count reveals that it consistently succeeds after only ``seeing" a fraction of the complete dataset---for instance, $100 \times 3500$ requires processing much less data than even one iteration of the full batch algorithm. This discrepancy becomes more pronounced as the problem scale increases. Unsurprisingly, then, the stochastic method scales much better as the size of the entire dataset $n$ grows in terms of real runtime.

\section{Discussion}
This paper proposes and analyzes a stochastic version of the recent proximal distance algorithm. The algorithm allows for a wide range of constraints to be considered in a projection-based framework. Our contributions now extend these merits to large-scale settings, and provide new theoretical insights. In particular, we establish convergence guarantees when the parameter $\rho$---and in turn the sequence of objectives---change over iterations. Such results were previously unavailable even in the non-stochastic, ``batch" version of the method. By leveraging tools from the analysis of stochastic and incremental schemes, we are able to derive new guarantees even when the MM geometry no longer holds due to subsampling. 

Our finite error analyses in the case of polynomial rate schedules  open the door to future directions for exponential and other rate schedules \citep{keys2019proximal}. In particular, experiments suggest that the conditions for convergence are sufficient but may not be necessary. The method may enjoy sharper analyses, 
perhaps by investigating the interplay between mini-batch size and convergence rate.  Finally, we focused our theoretical treatment to convex cases, but even the simple sparse example makes use of projection onto a non-convex constraint set with empirical success, and outperforms natural existing competitors. Indeed, prior work has reported similar successes \citep{landeros2022extensions}, and theoretical results on convergence are available in the fixed $\rho$ setting and batch case (see \citet{keys2019proximal}). A worthwhile direction for further investigation is to seek finite error bounds and convergence rate characterizations for non-convex problems. While we have found this problem challenging, techniques used in \citet{attouch2010proximal,karimi2016linear} may be promising to carry over into our projection-based framework, as an anonymous referee suggests. Modified versions of the method in future extensions of the work may also be more amenable to such analyses \citep{won2019projection}.

The findings in this article also translate to several practical takeaways. The method allows for constraints to be gracefully incorporated into a transparent algorithm that is easily scalable to large datasets via subsampling, but requires that the projections onto the desired constraint sets are practical. For those that do not come with closed form expressions, and as this subroutine is not one that is ameliorated via subsampling, users should be aware that iterative routines for computing these projections may form the computational bottleneck. Our theoretical and empirical results suggest that suboptimal choice of hyperparameters may slow the algorithm. However, the convergence rate analysis also newly provides guidance on the practical problem of appropriately specifying hyperparameters, which was previously handled only heuristically. Together with enabling the algorithm to scale with the data via a stochastic variant, we hope that these contributions make this recent method well-posed for practitioners on broad range of statistical applications under constraints.

\bibliographystyle{plainnat}
\bibliography{main}

\newpage
\appendix 
\small

\section*{Appendix}
\section{Proofs of Theoretical Results}
\subsection{Proof of Proposition \ref{prop_unique}}
\begin{proof}
For some $\bm{\theta}_0 \in C$, define $C' := C \cap \{\bm{\theta}: F(\bm{\theta}) \le F(\bm{\theta}_0)\}$. Because $F$ is continuous, strictly convex and coercive, $C'$ is a compact subset of $C$. Then, there exists a $\bm{\theta}_*\in C'$ such that $F(\bm{\theta}_*) = \min_{\bm{\theta} \in C'} F(\bm{\theta})$. Because $C'$ is constructed as the sublevel sets from $\bm{\theta}_0$ contained in $C$, it follows that $\bm{\theta}_* \in C$ is the minimizer of  $F$ in $C$.

Next to prove its uniqueness, suppose by contradiction that there can exist two distinct solutions $\bm{\theta}_1,\bm{\theta}_2\in C$, s.t. $F(\bm{\theta}_1)=F(\bm{\theta}_2)=\text{min}_{\bm{\theta}\in C}\,\,F(\bm{\theta})$. Since $C$ is convex, for any $\lambda\in (0,1)$, we have that  $\bm{\theta}_3=\lambda \bm{\theta}_1 + (1-\lambda)\bm{\theta}_2 \in C$. On the other hand, by strict convexity of $F$, \[ F(\bm{\theta}_3) < \lambda F(\bm{\theta}_1) + (1-\lambda) F(\bm{\theta}_2) = \text{min}_{\bm{\theta}\in C}\,\,F(\bm{\theta}). \] 
This is a contradiction; it thus follows that the solution $\bm{\theta}_* = \text{argmin}_{\bm{\theta}\in C}\,\,F(\bm{\theta})$ must be unique.
\end{proof}

\subsection{Proof of Proposition \ref{prop_bound}}
We first introduce several lemmas.

\begin{lemma}\label{make_prog}
 For a proximal mapping $\bm{x}_{k} = \bm{x}_{k-1} - \frac{1}{\rho_k} \nabla g(\bm{x}_{k})$, where the function $g$ is convex and differentiable, we have:
 
(a). $\lVert \nabla g(\bm{x}_k) \rVert \le \lVert \nabla g(\bm{x}_{k-1}) \rVert$

(b). $\lVert \bm{x}_{k} - \bm{y} \rVert^2 \le \lVert \bm{x}_{k-1} - \bm{y} \rVert^2 - \frac{2}{\rho_k}[g(\bm{x}_{k}) - g(\bm{y})] - \lVert \bm{x}_k - \bm{x}_{k-1} \rVert^2$ for any $\bm{y}$
\end{lemma}
Here result (a) is collected from Lemma 3.1 of \citep{https://doi.org/10.48550/arxiv.2206.12663} and result (b) is a variant of Proposition 1 in \citep{bertsekas2011incremental}. This lemma indicates that a proximal update based on $g$ is guaranteed to make progress with respect to $g$: (a) says that the gradient of the updated parameter will have smaller norm than before; (b) implies that the updated parameter will get closer to the minimal point if we choose $\bm{y}$ to be a minimum point of $g$.

\begin{proof} Result (b) can be seen upon expansion: 
\begin{equation*}
\begin{split}
\lVert \bm{x}_{k-1} - \bm{y} \rVert^2 &= \lVert \bm{x}_{k-1} - \bm{x}_{k} + \bm{x}_{k} - \bm{y} \rVert^2 \\
&= \lVert \bm{x}_{k-1} - \bm{x}_{k} \rVert^2 + \lVert \bm{x}_{k} - \bm{y} \rVert^2 + 2(\bm{x}_{k-1} - \bm{x}_{k})^T(\bm{x}_{k} - \bm{y}) \\
&= \lVert \bm{x}_{k-1} - \bm{x}_{k} \rVert^2 + \lVert \bm{x}_{k} - \bm{y} \rVert^2 + \frac{2}{\rho_k}\nabla g(\bm{x}_{k})^T(\bm{x}_{k} - \bm{y}) \\
&\ge \lVert \bm{x}_{k-1} - \bm{x}_{k} \rVert^2 + \lVert \bm{x}_{k} - \bm{y} \rVert^2 + \frac{2}{\rho_k}[g(\bm{x}_{k}) - g(\bm{y})] \qquad\text{(by convexity of $g$)} 
\end{split}    
\end{equation*}

Upon rearranging, the result follows:
\[  \lVert \bm{x}_{k} - \bm{y} \rVert^2 \; \le \;  \lVert \bm{x}_{k-1} - \bm{y} \rVert^2 - \frac{2}{\rho_k}[g(\bm{x}_{k}) - g(\bm{y})] - \lVert \bm{x}_{k-1} - \bm{x}_{k} \rVert^2 \] 
\end{proof}

\begin{lemma}\label{lemma_bound}
    If a non-negative sequence $\{y_0, y_1, ...\}$ satisfies that $y_k \le (1 + a_k) y_{k-1} + b_k$, where $a_k, b_k \ge 0$, $\sum_{i=1}^{+\infty} a_i < +\infty$ and $\sum_{i=1}^{+\infty} b_i < +\infty$. Then $\forall k$, 
$$y_k \le \exp(\sum_{i=1}^{+\infty} a_i) \cdot \sum_{i=1}^{+\infty} b_i + \exp(\sum_{i=1}^{+\infty} a_i) \cdot y_0. $$
\end{lemma}
\begin{proof}
By induction,
\begin{equation*}
\begin{split}
y_k \le& b_k + (1+a_k)y_{k-1} \\
\le & b_k + (1+a_k)b_{k-1} + (1+a_k)(1+a_{k-1})y_{k-2} \\
\le & b_k + (1+a_k)b_{k-1} + (1+a_k)(1+a_{k-1})b_{k-2} + (1+a_k)(1+a_{k-1})(1+a_{k-2})y_{k-3} \\
& ... \\
\le & b_k + (1+a_k)b_{k-1} + (1+a_k)(1+a_{k-1})b_{k-2} + ... + (1+a_k)(1+a_{k-1})...(1+a_2)b_1 \\
& +  (1+a_k)(1+a_{k-1})...(1+a_1)y_0 \\
\le & C_k \cdot [b_k + b_{k-1} + ... + b_1] + C_k \cdot y_0,
\end{split}
\end{equation*}
where $C_k = (1+a_k)(1+a_{k-1})...(1+a_1)$. We have $C_k \le \exp(\sum_{i=1}^{+\infty} a_i)$ because
\begin{equation*}
\begin{split}
\log{C_k} = \sum_{i=1}^k \log(1+a_i) \le \sum_{i=1}^k a_i \le \sum_{i=1}^{+\infty} a_i.
\end{split}
\end{equation*}
Therefore,
\begin{equation*}
\begin{split}
y_k \le& \exp(\sum_{i=1}^{+\infty} a_i) \cdot [b_k + b_{k-1} + ... + b_1] + \exp(\sum_{i=1}^{+\infty} a_i) \cdot y_0 \\
\le& \exp(\sum_{i=1}^{+\infty} a_i) \cdot \sum_{i=1}^{+\infty} b_i + \exp(\sum_{i=1}^{+\infty} a_i) \cdot y_0
\end{split}
\end{equation*}
\end{proof}

Now we are in the position to prove Proposition \ref{prop_bound}. 
\paragraph{Proof of Proposition \ref{prop_bound}}
\begin{proof}
We first prove result (a), beginning by seeking a recursion of the squared difference between the projected iterate and $\bm{\theta}_*$:
\begin{equation}\label{recursion}
\begin{split}
& \lVert P_C(\bm{\theta}_k) - \bm{\theta}_* \rVert^2  \\
\le& \lVert \bm{\theta}_k - \bm{\theta}_* \rVert^2 \qquad\text{(by non-expansive property)}\\
\le& \lVert P_C(\bm{\theta}_{k-1}) - \bm{\theta}_* \rVert^2 - \frac{2}{\rho_k}\{f(\bm{\theta}_k; \bm{z}_{\xi_k}) - f(\bm{\theta}_*; \bm{z}_{\xi_k})\} - \lVert \bm{\theta}_k - P_C(\bm{\theta}_{k-1}) \rVert^2 \qquad\text{\big(by Lemma \ref{make_prog} (b)\big)} \\
\le& \lVert P_C(\bm{\theta}_{k-1}) - \bm{\theta}_* \rVert^2 - \frac{2}{\rho_k}\{f[P_C(\bm{\theta}_{k-1}); \bm{z}_{\xi_k}] - f(\bm{\theta}_*; \bm{z}_{\xi_k})\} + \frac{2}{\rho_k}\{f[P_C(\bm{\theta}_{k-1}); \bm{z}_{\xi_k}] - f(\bm{\theta}_k; \bm{z}_{\xi_k})\} \\
\le& \lVert P_C(\bm{\theta}_{k-1}) - \bm{\theta}_* \rVert^2 - \frac{2}{\rho_k}\{f[P_C(\bm{\theta}_{k-1}); \bm{z}_{\xi_k}] - f(\bm{\theta}_*; \bm{z}_{\xi_k})\} + \frac{2}{\rho_k} \nabla f[P_C(\bm{\theta}_{k-1}); \bm{z}_{\xi_k}]^T[P_C(\bm{\theta}_{k-1}) - \bm{\theta}_k] \quad\text{(by convexity)} \\
=& \lVert P_C(\bm{\theta}_{k-1}) - \bm{\theta}_* \rVert^2 - \frac{2}{\rho_k}\{f[P_C(\bm{\theta}_{k-1}); \bm{z}_{\xi_k}] - f(\bm{\theta}_*; \bm{z}_{\xi_k})\} + \frac{2}{\rho_k^2} \nabla f[P_C(\bm{\theta}_{k-1}); \bm{z}_{\xi_k}]^T \nabla f(\bm{\theta}_k; \bm{z}_{\xi_k}) \\
\le& \lVert P_C(\bm{\theta}_{k-1}) - \bm{\theta}_* \rVert^2 - \frac{2}{\rho_k}\{f[P_C(\bm{\theta}_{k-1}); \bm{z}_{\xi_k}] - f(\bm{\theta}_*; \bm{z}_{\xi_k})\} + \frac{2}{\rho_k^2} \lVert \nabla f[P_C(\bm{\theta}_{k-1}); \bm{z}_{\xi_k}] \rVert^2,
\end{split}
\end{equation}
where the last inequality is due to Cauchy–Schwarz inequality and Lemma \ref{make_prog} (a). Also, due to triangular inequality and Assumption \ref{lsmooth}, we have
\begin{equation}\label{triangle}
\begin{split}
& \lVert \nabla f[P_C(\bm{\theta}_{k-1}); \bm{z}_{\xi_k}] \rVert - \lVert \nabla f(\bm{\theta}_*; \bm{z}_{\xi_k}) \rVert \le \lVert \nabla f[P_C(\bm{\theta}_{k-1}); \bm{z}_{\xi_k}] - \nabla f(\bm{\theta}_*; \bm{z}_{\xi_k}) \rVert \le L \cdot \lVert P_C(\bm{\theta}_{k-1}) - \bm{\theta}_* \rVert \\
\Rightarrow& \lVert \nabla f[P_C(\bm{\theta}_{k-1}); \bm{z}_{\xi_k}] \rVert \le \lVert \nabla f(\bm{\theta}_*; \bm{z}_{\xi_k}) \rVert + L \lVert P_C(\bm{\theta}_{k-1}) - \bm{\theta}_* \rVert \\
\Rightarrow& \lVert \nabla f[P_C(\bm{\theta}_{k-1}); \bm{z}_{\xi_k}] \rVert^2 \le 2(\lVert \nabla f(\bm{\theta}_*; \bm{z}_{\xi_k}) \rVert^2 + L^2 \lVert P_C(\bm{\theta}_{k-1}) - \bm{\theta}_* \rVert^2),
\end{split}
\end{equation}
where the last inequality is because $(a+b)^2 \le 2(a^2+b^2)$. Plugging \eqref{triangle} into \eqref{recursion} leads to:
\begin{equation*}
\begin{split}
\lVert P_C(\bm{\theta}_k) -\bm{\theta}_* \rVert^2 \quad \le \quad (1 + \frac{4L^2}{\rho_k^2}) \lVert P_C(\bm{\theta}_{k-1}) -\bm{\theta}_* \rVert^2 - \frac{2}{\rho_k}\{f[P_C(\bm{\theta}_{k-1}); \bm{z}_{\xi_k}] - f(\bm{\theta}_*; \bm{z}_{\xi_k})\} + \frac{4}{\rho_k^2} \lVert \nabla f(\bm{\theta}_*; \bm{z}_{\xi_k}) \rVert^2,
\end{split}
\end{equation*}
Taking expectation on both sides leads to:
\begin{equation*}
\begin{split}
\mathbb{E}\{ \lVert P_C(\bm{\theta}_k) -\bm{\theta}_* \rVert^2\} \; &\le \; (1 + \frac{4L^2}{\rho_k^2}) \mathbb{E}\{\lVert P_C(\bm{\theta}_{k-1}) -\bm{\theta}_* \rVert^2\} - \frac{2}{\rho_k}\mathbb{E}\{F[P_C(\bm{\theta}_{k-1})] - F(\bm{\theta}_*)\} + \frac{4}{\rho_k^2} \mathbb{E}\{\lVert \nabla f(\bm{\theta}_*; \bm{z}_{\xi_k}) \rVert^2\} \\
&\le \; (1 + \frac{4L^2}{\rho_k^2}) \mathbb{E}\{\lVert P_C(\bm{\theta}_{k-1}) -\bm{\theta}_* \rVert^2\} + \frac{4}{\rho_k^2} \mathbb{E}\{\lVert \nabla f(\bm{\theta}_*; \bm{z}_{\xi_k}) \rVert^2\} \\
&= \; (1 + \frac{4L^2}{\rho_1^2 \cdot k^{2\gamma}}) \mathbb{E}\{\lVert P_C(\bm{\theta}_{k-1}) -\bm{\theta}_* \rVert^2\} + \frac{4G^2}{\rho_1^2 \cdot k^{2\gamma}},
\end{split}
\end{equation*}
where the second term on the RHS of the first line is obtained by taking conditional expectation given $\mathcal{F}_{k-1}$ and then averaging over $\mathcal{F}_{k-1}$, the penultimate step is because that $\bm{\theta}_*$ is the constrained solution, and the last inequality is obtained by plugging in $\rho_k = \rho_1 \cdot k^{\gamma}$ and $G^2:= \mathbb{E}\{\lVert \nabla f(\bm{\theta}_*; \bm{z}_{\xi_k}) \rVert^2\} = \frac{1}{n} \sum_{i=1}^{n} \lVert \nabla f(\bm{\theta}_*; \bm{z}_i) \rVert^2$. Now we can invoke Lemma \ref{lemma_bound} and obtain conclusion (a). To prove conclusion (b), we only need to take expectation on both sides of \eqref{triangle}:
\begin{equation*}
\begin{split}
\mathbb{E}\{ \lVert \nabla f[P_C(\bm{\theta}_{k}); \bm{z}_{\xi_k}] \rVert^2 \} &\le 2(G^2 + r^2 L^2) := c^2.
\end{split}
\end{equation*}

\end{proof}

Based on the above Lemmas and Propositions, next we will prove the main Theorems. 

\subsection{Proof of Theorem \ref{prop_conv}}
We prove a similar result to Proposition 9 of \citet{bertsekas2011incremental} under strictly weaker assumptions. The difference is that we do not require that $\lVert \nabla f[P_C(\bm{\theta}_{k-1}); \bm{z}_{\xi_k}] \rVert$ is bounded almost surely (see Assumption 4 (b) in \citep{bertsekas2011incremental}). Instead, we assume only that the gradient of each component function is $L$-Lipschitz, which does not require the gradients to be bounded. Under this assumption, we proved Proposition \ref{prop_bound} that plays a similar role as the Assumption 4 (b) in \citep{bertsekas2011incremental}.

We first recall the Supermartingale Convergence Theorem.
\begin{lemma}\label{supermart}(Supermartingale Convergence Theorem \citep{robbins1971convergence, bertsekas2011incremental}) Let $Y_k$, $\bm{z}_k$ and $W_k$, $k=0,1,...$, be three sequences of random variables and let $\mathcal{F}_k$, $k=0,1.,..$, be sets of random variables such that $\mathcal{F}_k\subset \mathcal{F}_{k+1}$ for all $k$. Suppose that:
\\
(1) The random variables $Y_k$, $\bm{z}_k$ and $W_k$ are nonnegative, and are functions of the random variables in $\mathcal{F}_k$.
\\
(2) For each $k$, $\mathbb{E}\{Y_{k+1}|\mathcal{F}_k\}\le Y_k - \bm{z}_k + W_k$.
\\
(3) There holds, with probability 1, $\sum_{k=0}^\infty W_k < \infty$.
\\
Then we have $\sum_{k=0}^\infty \bm{z}_k < \infty$, and the sequence $Y_k$ converges to a nonnegative random variable $Y$, with probability 1.
\end{lemma}

\paragraph{Proof of Theorem \ref{prop_conv}}
\begin{proof}
We begin by bounding the squared difference between the projected iterate and $\bm{\theta}_\ast$, and we can start from Eq. \eqref{recursion}:
\begin{equation*}
\begin{split}
\lVert P_C(\bm{\theta}_k) - \bm{\theta}_* \rVert^2 &\le \lVert P_C(\bm{\theta}_{k-1}) - \bm{\theta}_* \rVert^2 - \frac{2}{\rho_k}\{f[P_C(\bm{\theta}_{k-1}); \bm{z}_{\xi_k}] - f(\bm{\theta}_*; \bm{z}_{\xi_k})\} + \frac{2}{\rho_k^2} \lVert \nabla f[P_C(\bm{\theta}_{k-1}); \bm{z}_{\xi_k}] \rVert^2
\end{split}
\end{equation*}

Taking expectations conditional on the filtration $\mathcal{F}_{k-1}$ on both sides leads to:
\begin{equation}\label{eq:super}
\begin{split}
&\quad\; \mathbb{E}\{\lVert P_C(\bm{\theta}_k) - \bm{\theta}_*\rVert^2 \mid \mathcal{F}_{k-1}\}  \\
&\le \lVert P_C(\bm{\theta}_{k-1}) - \bm{\theta}_*\rVert^2 -\frac{2}{\rho_k}\{F[P_C(\bm{\theta}_{k-1})] - F(\bm{\theta}_*)\} + \frac{2}{\rho_k^2} \mathbb{E}\{ \lVert \nabla f[P_C(\bm{\theta}_{k-1}); \bm{z}_{\xi_k}] \rVert^2 \mid \mathcal{F}_{k-1} \}
\end{split}
\end{equation}
The last term on the RHS satisfies condition (3) in Lemma \ref{supermart}: $\sum_{k=1}^{\infty} \frac{2}{\rho_k^2} \mathbb{E}\{ \lVert \nabla f[P_C(\bm{\theta}_{k-1}); \bm{z}_{\xi_k}] \rVert^2 \mid \mathcal{F}_{k-1} \} < \infty$ with probability $1$, because for any $M>0$,
\begin{equation*}
\begin{split}
&\mathbb{P}(\sum_{k=1}^{\infty} \frac{2}{\rho_k^2} \mathbb{E}\{ \lVert \nabla f[P_C(\bm{\theta}_{k-1}); \bm{z}_{\xi_k}] \rVert^2 \mid \mathcal{F}_{k-1} \} \ge M) \\
\le& \frac{\mathbb{E}\bigl\{ \sum_{k=1}^{\infty} \frac{2}{\rho_k^2} \mathbb{E}\{ \lVert \nabla f[P_C(\bm{\theta}_{k-1}); \bm{z}_{\xi_k}] \rVert^2 \mid \mathcal{F}_{k-1} \} \bigr\}}{M} \qquad\text{(by Markov's inequality)}\\
=& \frac{\sum_{k=1}^{\infty} \frac{2}{\rho_k^2} \mathbb{E}\bigl\{ \mathbb{E}\{ \lVert \nabla f[P_C(\bm{\theta}_{k-1}); \bm{z}_{\xi_k}] \rVert^2 \mid \mathcal{F}_{k-1} \} \bigr\}}{M} \\
=& \frac{\sum_{k=1}^{\infty} \frac{2}{\rho_k^2} \mathbb{E}\{ \lVert \nabla f[P_C(\bm{\theta}_{k-1}); \bm{z}_{\xi_k}] \rVert^2 \} }{M} \\
\le& \frac{\sum_{k=1}^{\infty} \frac{2c^2}{\rho_k^2} }{M} \qquad\text{\big(by Proposition \ref{prop_bound} (b), and the numerator is finite because of Assumption \ref{rm_cond}\big)} \\
\Rightarrow \qquad& 1\ge \mathbb{P}(\sum_{k=1}^{\infty} \frac{2}{\rho_k^2} \mathbb{E}\{ \lVert \nabla f[P_C(\bm{\theta}_{k-1}); \bm{z}_{\xi_k}] \rVert^2 \mid \mathcal{F}_{k-1} \} < M) \ge 1 - \frac{\sum_{k=1}^{\infty} \frac{2c^2}{\rho_k^2} }{M} \rightarrow 1, \text{ as } M\rightarrow +\infty \\
\Rightarrow \qquad& \mathbb{P}(\sum_{k=1}^{\infty} \frac{2}{\rho_k^2} \mathbb{E}\{ \lVert \nabla f[P_C(\bm{\theta}_{k-1}); \bm{z}_{\xi_k}] \rVert^2 \mid \mathcal{F}_{k-1} \} < +\infty) = 1 
\end{split}
\end{equation*}

Now, all the conditions in Lemma \ref{supermart} are met and we can invoke the supermartingale convergence theorem with (\ref{eq:super}). In particular, this implies that with probability $1$, $\lVert P_C(\bm{\theta}_k) - \bm{\theta}_* \rVert^2$ converges to some nonnegative random variable $Y$, and \begin{equation}\label{eq:finite} \sum_{k=1}^{\infty}\frac{2}{\rho_k}\{F[P_C(\bm{\theta}_{k-1})] - F(\bm{\theta}_*)\} \; < \; \infty . \end{equation}

From here onward, one can establish the result by following an identical argument in Proposition 9 of \citet{bertsekas2011incremental}. We instead provide a simpler, more direct proof for completeness. 

To this end, we first show that with probability $1$, there exists a subsequence indexed by $\{i_k\}_{k=1}^{+\infty}$ such that $F[P_C(\bm{\theta}_{i_k - 1})] - F(\bm{\theta}_*) \rightarrow 0$. Assume by contradiction that with positive probability, no subsequences of $F[P_C(\bm{\theta}_{k - 1})] - F(\bm{\theta}_*)$ converge to $0$.
This implies that we can find  $\delta > 0$ such that only finitely many $k$ satisfy  $F[P_C(\bm{\theta}_{k - 1})] - F(\bm{\theta}_*) < \delta$. Indeed, otherwise for all $\delta>0$, the condition $F[P_C(\bm{\theta}_{k - 1})] - F(\bm{\theta}_*) < \delta$ is satisfied for infinitely many $k$, so that we could have chosen such a convergent subsequence. 

However, if $F[P_C(\bm{\theta}_{k - 1})] - F(\bm{\theta}_*) < \delta$ only holds for finitely many $k$, then there are infinitely many terms in the sum that contribute positive mass. Since $\sum_{k=1}^\infty \frac{1}{\rho_k}$ is divergent too,
\begin{equation*}\sum_{k=1}^{\infty}\frac{2}{\rho_k}\{F[P_C(\bm{\theta}_{k-1})] - F(\bm{\theta}_*)\} \; = \; \infty, \end{equation*}
a contradiction of \eqref{eq:finite}.

Therefore, it must hold with probability $1$ that there exists a subsequence $F[P_C(\bm{\theta}_{i_k - 1})] - F(\bm{\theta}_*) \rightarrow 0$. As a result, with probability $1$, the same limit $Y$ given by the supermartingale theorem is shared by the subsequence indexed  $\{i_k\}_{k=1}^{+\infty}$: 
$$
\lVert P_C(\bm{\theta}_{i_k - 1}) - \bm{\theta}_* \rVert^2 \rightarrow Y.
$$
It remains to transfer a similar result onto the iterate sequence. 
Let $\omega \in \Omega$, where we denote the probability space $(\Omega, \mathcal{F}, \mathcal{P})$. 
Since $F$ is \textit{strictly} convex and $\bm{\theta}_\ast$ is the (constrained) minimizer, we have
$$
\lim_{k\rightarrow \infty} \lVert P_C\bigl(\bm{\theta}_{i_k - 1}(\omega)\bigr) - \bm{\theta}_* \rVert^2 > 0 \, \Rightarrow \, \lim_{k\rightarrow \infty} F[P_C\bigl(\bm{\theta}_{i_k - 1}(\omega)\bigr)] - F(\bm{\theta}_*) > 0.
$$
It follows that
$$
\mathbb{P}(\lim_{k\rightarrow \infty} \lVert P_C(\bm{\theta}_{i_k - 1}) - \bm{\theta}_* \rVert^2 > 0) \, \le \, \mathbb{P}\bigl(\lim_{k\rightarrow \infty} F[P_C(\bm{\theta}_{i_k - 1})] - F(\bm{\theta}_*) > 0\bigr)  = 0.
$$
That is, with probability $1$ the limit $Y=0$, and as a result the entire iterate sequence $P_C(\bm{\theta}_k) \rightarrow \bm{\theta}_*$ almost surely.

\end{proof}

\subsection{Proof of Theorem \ref{finite_para}}
We first introduce a useful lemma.
\begin{lemma}\label{lemma_toulis} (Combining Lemma 2 and Corollary 1 in \citep{toulis2015proximal}) For sequences $a_n = a_1 n^{-\alpha}$ and $b_n = b_1 n^{-\beta}$, where $\alpha > \beta$, and $a_1, b_1, \beta > 0$ and $1<\alpha<1+\beta$. Define,
$$
\delta_n \triangleq \frac{1}{a_n}(a_{n-1}/b_{n-1} - a_n / b_n).
$$
\end{lemma}
If a positive sequence $y_n > 0$ satisfies the recursive inequality
$$
y_n \le \frac{1}{1+b_n}y_{n-1} + a_n.
$$
Then, 
$$
y_n \le 2\frac{a_1 (1 + b_1)}{b_1}n^{-\alpha + \beta} + \exp(-\log(1+b_1) \phi_{\beta}(n))[y_0 + (1+b_1)^{n_0} A],
$$
where $A = \sum_{i=1}^{\infty} a_i$, $\phi_{\beta}(n) = n^{1-\beta}$ if $\beta \in (0.5, 1)$ and $\phi_{\beta}(n) = \log n$ if $\beta = 1$, and $n_0$ is a positive integer satisfying that $\delta_{n_0} < 1$.

Now we begin to prove Theorem \ref{finite_para}.
\paragraph{Proof of Theorem \ref{finite_para}}

\begin{proof}
Our objective is to seek a recursion for $\mathbb{E}\{\lVert P_C(\bm{\theta}_k) - \bm{\theta}_* \rVert^2\}$. We first borrow a nice idea from recent work by \citet{https://doi.org/10.48550/arxiv.2206.12663}, which entails rewriting our implicit scheme as an explicit update plus a remainder term $R_k$. This allows for a cleaner analysis, showing that the terms involving $R_k$ go to zero sufficiently fast. To this end, we begin by rewriting the updating rule as
\begin{equation*}
\begin{split}
\bm{\theta}_k &= P_C(\bm{\theta}_{k-1}) - \frac{1}{\rho_k}\nabla f(\bm{\theta}_k; \bm{z}_{\xi_k}) \\
&= P_C(\bm{\theta}_{k-1}) - \frac{1}{\rho_k}\nabla f[P_C(\bm{\theta}_{k-1}); \bm{z}_{\xi_k}] + R_k,
\end{split}
\end{equation*}
where \[ R_k = \frac{1}{\rho_k} \{ \nabla f[P_C(\bm{\theta}_{k-1}); \bm{z}_{\xi_k}] - \nabla f(\bm{\theta}_k; \bm{z}_{\xi_k}) \}. \] 
The motivation for doing this rewrite is the fact mentioned in the main paper that $\mathbb{E}\{\nabla f(\bm{\theta}_k; \bm{z}_{\xi_k}) \mid \mathcal{F}_{k-1}\} \ne  \nabla F(\bm{\theta}_k)$, but \[ \mathbb{E}\{\nabla f[P_C(\bm{\theta}_{k-1}); \bm{z}_{\xi_k}] \mid \mathcal{F}_{k-1}\} =  \nabla F[P_C(\bm{\theta}_{k-1})] . \] This latter equation enables us to use the fact that $F$ is strongly convex at $\bm{\theta}_\ast$ (Assumption \ref{strong_conv}): we have
\begin{equation}\label{expand}
\begin{split}
&\mathbb{E}\{\lVert P_C(\bm{\theta}_k) - \bm{\theta}_* \rVert^2\} \\
\le& \mathbb{E}\{\lVert \bm{\theta}_k - \bm{\theta}_* \rVert^2\} \qquad\text{(by non-expansive property)}\\
=& \mathbb{E}\{\lVert P_C(\bm{\theta}_{k-1}) - \bm{\theta}_* \rVert^2\} + \frac{1}{\rho_k^2} \mathbb{E}\{\lVert \nabla f[P_C(\bm{\theta}_{k-1}); \bm{z}_{\xi_k}] \rVert^2\} + \mathbb{E}\{\lVert R_k \rVert^2\} \\
& -\frac{2}{\rho_k}\mathbb{E}\{ \nabla f[P_C(\bm{\theta}_{k-1}); \bm{z}_{\xi_k}]^T [P_C(\bm{\theta}_{k-1}) - \bm{\theta}_*] \} \\
& -\frac{2}{\rho_k} \mathbb{E}\{ \nabla f[P_C(\bm{\theta}_{k-1}); \bm{z}_{\xi_k}]^T R_k \}  + 2\mathbb{E}\{[P_C(\bm{\theta}_{k-1}) - \bm{\theta}_*]^T R_k\}.
\end{split}
\end{equation}

Then we proceed to bound each term on the RHS of \eqref{expand}. The second term can be bounded by Proposition \ref{prop_bound} (b):
\begin{equation*}
\begin{split}
\frac{1}{\rho_k^2} \mathbb{E}\{\lVert \nabla f[P_C(\bm{\theta}_{k-1}); \bm{z}_{\xi_k}] \rVert^2\} 
&\le \frac{c^2}{\rho_k^2}
\end{split}
\end{equation*}
The third term:
\begin{equation*}
\begin{split}
\mathbb{E}\{\lVert R_k \rVert^2\} &= \frac{1}{\rho_k^2} \mathbb{E}\{\ \lVert \nabla f[P_C(\bm{\theta}_{k-1}); \bm{z}_{\xi_k}] - \nabla f(\bm{\theta}_k; \bm{z}_{\xi_k}) \rVert^2 \} \\
&= \frac{1}{\rho_k^2} \mathbb{E}\{\ \lVert \nabla f[P_C(\bm{\theta}_{k-1}); \bm{z}_{\xi_k}] \rVert^2 +\lVert \nabla f(\bm{\theta}_k; \bm{z}_{\xi_k}) \rVert^2 - 2 \nabla f[P_C(\bm{\theta}_{k-1}); \bm{z}_{\xi_k}]^T \nabla f(\bm{\theta}_k; \bm{z}_{\xi_k})\} \\
&\le \frac{4}{\rho_k^2} \mathbb{E}\{ \lVert \nabla f[P_C(\bm{\theta}_{k-1}); \bm{z}_{\xi_k}] \rVert^2 \} \qquad\text{\big(by Cauchy–Schwarz inequality and Lemma \ref{make_prog} (a)\big)}\\
&\le \frac{4c^2}{\rho_k^2} \text{ \big(by Proposition \ref{prop_bound} (b)\big)}
\end{split}
\end{equation*}
The fourth term:
\begin{equation*}
\begin{split}
&\mathbb{E}\{ \nabla f[P_C(\bm{\theta}_{k-1}); \bm{z}_{\xi_k}]^T [P_C(\bm{\theta}_{k-1}) - \bm{\theta}_*] \} \\
=& \mathbb{E} \bigl\{\mathbb{E}\{ \nabla f[P_C(\bm{\theta}_{k-1}); \bm{z}_{\xi_k}]^T [P_C(\bm{\theta}_{k-1}) - \bm{\theta}_*] \mid \mathcal{F}_{k-1} \}\bigr\} \\
=& \mathbb{E}\{ \nabla F[P_C(\bm{\theta}_{k-1})]^T [P_C(\bm{\theta}_{k-1}) - \bm{\theta}_*] \} \\
\ge& \mathbb{E}\{ F[P_C(\bm{\theta}_{k-1})] - F(\bm{\theta}_*) + \frac{\mu}{2} \lVert P_C(\bm{\theta}_{k-1}) - \bm{\theta}_* \rVert^2 \} \qquad\text{(by Assumption \ref{strong_conv})} \\
\ge& \frac{\mu}{2} \mathbb{E}\{ \lVert P_C(\bm{\theta}_{k-1}) - \bm{\theta}_* \rVert^2 \} \qquad\text{(since $\bm{\theta}_*$ is the constrained solution)} \\
&\Rightarrow -\frac{2}{\rho_k} \mathbb{E}\{ \nabla f[P_C(\bm{\theta}_{k-1}); \bm{z}_{\xi_k}]^T [P_C(\bm{\theta}_{k-1}) - \bm{\theta}_*] \} \le -\frac{\mu}{\rho_k} \mathbb{E}\{ \lVert P_C(\bm{\theta}_{k-1}) - \bm{\theta}_* \rVert^2 \}
\end{split}
\end{equation*}

The fifth term:
\begin{equation*}
\begin{split}
&\mathbb{E}\{ \nabla f [P_C(\bm{\theta}_{k-1}); \bm{z}_{\xi_k}]^T R_k \} \\
=& \frac{1}{\rho_k} \mathbb{E}\{\lVert \nabla f[P_C(\bm{\theta}_{k-1}); \bm{z}_{\xi_k}] \rVert^2 - \nabla f[P_C(\bm{\theta}_{k-1}); \bm{z}_{\xi_k}]^T \nabla f(\bm{\theta}_{k}; \bm{z}_{\xi_k}) \} \\
\ge& \frac{1}{\rho_k} \mathbb{E}\{\lVert \nabla f[P_C(\bm{\theta}_{k-1}); \bm{z}_{\xi_k}] \rVert^2 - \lVert \nabla f[P_C(\bm{\theta}_{k-1}); \bm{z}_{\xi_k}] \rVert \cdot \lVert \nabla f(\bm{\theta}_{k}; \bm{z}_{\xi_k}) \rVert \} \qquad\text{(by Cauchy–Schwarz inequality)}\\
\ge& 0 \qquad\text{\big(by Lemma \ref{make_prog} (a)\big)}\\
&\Rightarrow -\frac{2}{\rho_k} \mathbb{E}\{ \nabla f [P_C(\bm{\theta}_{k-1}); \bm{z}_{\xi_k}]^T R_k \} \le 0
\end{split}
\end{equation*}

Finally, to bound the sixth term:
\begin{equation*}
\begin{split}
& 2\mathbb{E}\{ [P_C(\bm{\theta}_{k-1}) - \bm{\theta}_*]^T R_k \} \\
= & \frac{2}{\rho_k} \mathbb{E}\bigl\{ [P_C(\bm{\theta}_{k-1}) - \bm{\theta}_*]^T \{ \nabla f[P_C(\bm{\theta}_{k-1}); \bm{z}_{\xi_k}] - \nabla f(\bm{\theta}_k; \bm{z}_{\xi_k}) \} \bigr\} \\
\le& \frac{2}{\rho_k} \mathbb{E}\{ \lVert P_C(\bm{\theta}_{k-1}) - \bm{\theta}_* \rVert \cdot \lVert \nabla f[P_C(\bm{\theta}_{k-1}); \bm{z}_{\xi_k}] - \nabla f(\bm{\theta}_k; \bm{z}_{\xi_k}) \rVert \} \qquad\text{(by Cauchy–Schwarz inequality)}\\
\le& \frac{2L}{\rho_k} \mathbb{E}\{ \lVert P_C(\bm{\theta}_{k-1}) - \bm{\theta}_* \rVert \cdot \lVert P_C(\bm{\theta}_{k-1}) - \bm{\theta}_k \rVert \} \qquad\text{(by Assumption \ref{lsmooth})} \\
=& \frac{2L}{\rho_k^2} \mathbb{E}\{ \lVert P_C(\bm{\theta}_{k-1}) - \bm{\theta}_* \rVert \cdot \lVert \nabla f(\bm{\theta}_k; \bm{z}_{\xi_k}) \rVert \} \\
\le& \frac{2L}{\rho_k^2} \mathbb{E}\{ \lVert P_C(\bm{\theta}_{k-1}) - \bm{\theta}_* \rVert \cdot \lVert \nabla f[P_C(\bm{\theta}_{k-1}); \bm{z}_{\xi_k}] \rVert \} \qquad\text{\big(by Lemma \ref{make_prog} (a)\big)}.
\end{split}
\end{equation*}
From the second line of Eq. \eqref{triangle}, we have
\begin{equation*}
\begin{split}
\lVert P_C(\bm{\theta}_{k-1}) - \bm{\theta}_* \rVert \cdot \lVert \nabla f[P_C(\bm{\theta}_{k-1}); \bm{z}_{\xi_k}] \rVert 
\quad \le \quad L \cdot \lVert P_C(\bm{\theta}_{k-1}) - \bm{\theta}_* \rVert^2 + \lVert \nabla f(\bm{\theta}_*; \bm{z}_{\xi_k}) \rVert \cdot \lVert P_C(\bm{\theta}_{k-1}) - \bm{\theta}_* \rVert.
\end{split}
\end{equation*}
Then, the sixth term
\begin{equation*}
\begin{split}
    & 2\mathbb{E}\{ [P_C(\bm{\theta}_{k-1}) - \bm{\theta}_*]^T R_k \} \\
    \le& \frac{2L^2}{\rho_k^2} \mathbb{E}\{ \lVert P_C(\bm{\theta}_{k-1}) - \bm{\theta}_* \rVert^2 \} + \frac{2L}{\rho_k^2} \mathbb{E}\{ \lVert \nabla f(\bm{\theta}_*; \bm{z}_{\xi_k}) \rVert \cdot \lVert P_C(\bm{\theta}_{k-1}) - \bm{\theta}_* \rVert \} \\
    =& \frac{2L^2}{\rho_k^2} \mathbb{E}\{ \lVert P_C(\bm{\theta}_{k-1}) - \bm{\theta}_* \rVert^2 \} + \frac{2L}{\rho_k^2} \mathbb{E}\{ \lVert \nabla f(\bm{\theta}_*; \bm{z}_{\xi_k}) \rVert \} \cdot \mathbb{E}\{ \lVert P_C(\bm{\theta}_{k-1}) - \bm{\theta}_* \rVert \} \qquad \text{(two independent terms)}\\
    \le& \frac{2L^2 r^2 + 2LGr}{\rho_k^2},
\end{split}
\end{equation*}
where the last inequality is according to Proposition \ref{prop_bound} (a) and Jensen's inequality.

Organizing these bounds together, we obtain the following recursion:

\begin{equation*}
\begin{split}
&\mathbb{E}\{ \lVert P_C(\bm{\theta}_k) - \bm{\theta}_* \rVert^2 \} \\
\le& (1-\frac{\mu}{\rho_k}) \mathbb{E}\{ \lVert P_C(\bm{\theta}_{k-1}) - \bm{\theta}_* \rVert^2 \} + \frac{1}{\rho_k^2} (5c^2 + 2L^2 r^2 + 2LGr) \\
\le& \frac{1}{1 + \frac{\mu}{\rho_k}}\mathbb{E}\{ \lVert P_C(\bm{\theta}_{k-1}) - \bm{\theta}_* \rVert^2 \} + \frac{1}{\rho_k^2} (5c^2 + 2L^2 r^2 + 2LGr) \\
=& \frac{1}{1 + \frac{\mu}{\rho_1 k^{\gamma}}}\mathbb{E}\{ \lVert P_C(\bm{\theta}_{k-1}) - \bm{\theta}_* \rVert^2 \} + \frac{1}{\rho_1^2 k^{2\gamma}} (5c^2 + 2L^2 r^2 + 2LGr),
\end{split}
\end{equation*}
where the last step is due to plugging in $\rho_k = \rho_1 k^{\gamma}$, and the penultimate step is because that $(1-\frac{\mu}{\rho_k}) (1+\frac{\mu}{\rho_k}) < 1$. Now we can invoke Lemma \ref{lemma_toulis} and the result follows.
\end{proof}

\subsection{Proof of Theorem \ref{finite_F}}

\begin{proof}
We start from Eq. \eqref{recursion}:
\begin{equation*}
\begin{split}
\lVert P_C(\bm{\theta}_k) - \bm{\theta}_* \rVert^2 \le \lVert P_C(\bm{\theta}_{k-1}) - \bm{\theta}_* \rVert^2 - \frac{2}{\rho_k}\{f[P_C(\bm{\theta}_{k-1}); \bm{z}_{\xi_k}] - f(\bm{\theta}_*; \bm{z}_{\xi_k})\} + \frac{2}{\rho_k^2} \lVert \nabla f[P_C(\bm{\theta}_{k-1}); \bm{z}_{\xi_k}] \rVert^2
\end{split}
\end{equation*}

Upon rearranging and further manipulation, we obtain
\begin{equation*}
\begin{split}
&\quad\; \frac{2}{\rho_k}\{f[P_C(\bm{\theta}_{k-1}); \bm{z}_{\xi_k}] - f(\bm{\theta}_*; \bm{z}_{\xi_k})\} \\
&\le \lVert P_C(\bm{\theta}_{k-1}) - \bm{\theta}_* \rVert^2 - \lVert P_C(\bm{\theta}_{k}) - \bm{\theta}_* \rVert^2 + \frac{2}{\rho_k^2} \lVert \nabla f[P_C(\bm{\theta}_{k-1}); \bm{z}_{\xi_k}] \rVert^2\\
&= \lVert P_C(\bm{\theta}_{k-1}) - \bm{\theta}_* \rVert^2 - \lVert P_C(\bm{\theta}_{k}) -P_C(\bm{\theta}_{k-1}) + P_C(\bm{\theta}_{k-1}) - \bm{\theta}_* \rVert^2 + \frac{2}{\rho_k^2} \lVert \nabla f[P_C(\bm{\theta}_{k-1}); \bm{z}_{\xi_k}] \rVert^2 \\
&= -\lVert P_C(\bm{\theta}_k) - P_C(\bm{\theta}_{k-1}) \rVert^2 - 2[P_C(\bm{\theta}_k) - P_C(\bm{\theta}_{k-1})]^T[P_C(\bm{\theta}_{k-1}) - \bm{\theta}_*] + \frac{2}{\rho_k^2} \lVert \nabla f[P_C(\bm{\theta}_{k-1}); \bm{z}_{\xi_k}] \rVert^2   \\
&\le 2\lVert P_C(\bm{\theta}_k) - P_C(\bm{\theta}_{k-1})\rVert \cdot \lVert P_C(\bm{\theta}_{k-1}) - \bm{\theta}_* \rVert + \frac{2}{\rho_k^2} \lVert \nabla f[P_C(\bm{\theta}_{k-1}); \bm{z}_{\xi_k}] \rVert^2 \qquad \text{(by Cauchy–Schwarz inequality)}\\
&\le 2\lVert \bm{\theta}_k - P_C(\bm{\theta}_{k-1})\rVert \cdot \lVert P_C(\bm{\theta}_{k-1}) - \bm{\theta}_* \lVert + \frac{2}{\rho_k^2} \lVert \nabla f[P_C(\bm{\theta}_{k-1}); \bm{z}_{\xi_k}] \rVert^2 \qquad\text{(by nonexpansiveness)}\\
&= \frac{2}{\rho_k}\lVert \nabla f(\bm{\theta}_k; \bm{z}_{\xi_k}) \rVert \cdot \lVert P_C(\bm{\theta}_{k-1}) - \bm{\theta}_* \rVert + \frac{2}{\rho_k^2} \lVert \nabla f[P_C(\bm{\theta}_{k-1}); \bm{z}_{\xi_k}] \rVert^2 \\
&\le \frac{2}{\rho_k}\lVert \nabla f[P_C(\bm{\theta}_{k-1}); \bm{z}_{\xi_k}] \rVert \cdot \lVert P_C(\bm{\theta}_{k-1}) - \bm{\theta}_* \rVert + \frac{2}{\rho_k^2} \lVert \nabla f[P_C(\bm{\theta}_{k-1}); \bm{z}_{\xi_k}] \rVert^2 \qquad\text{\big(by Lemma \ref{make_prog} (a)\big)}
\end{split} \end{equation*}

Rearranging again and cancelling constants reveals
\begin{equation*}
\begin{split}
&\quad\; f[P_C(\bm{\theta}_{k-1}); \bm{z}_{\xi_k}] - f(\bm{\theta}_*; \bm{z}_{\xi_k}) \\
&\le \lVert \nabla f[P_C(\bm{\theta}_{k-1}); \bm{z}_{\xi_k}] \rVert \cdot \lVert P_C(\bm{\theta}_{k-1}) - \bm{\theta}_* \rVert + \frac{1}{\rho_k} \cdot \lVert \nabla f[P_C(\bm{\theta}_{k-1}); \bm{z}_{\xi_k}] \rVert^2 \\
&\le  L \cdot \lVert P_C(\bm{\theta}_{k-1}) - \bm{\theta}_* \rVert^2 + \lVert \nabla f(\bm{\theta}_*; \bm{z}_{\xi_k}) \rVert \cdot \lVert P_C(\bm{\theta}_{k-1}) - \bm{\theta}_* \rVert + \frac{1}{\rho_k} \cdot \lVert \nabla f[P_C(\bm{\theta}_{k-1}); \bm{z}_{\xi_k}] \rVert^2,
\end{split}
\end{equation*}
where the last inequality is due to the second line of Eq. \eqref{triangle}. Taking conditional expectations given $\mathcal{F}_{k-1}$ on both sides leads to
\begin{equation*}
\begin{split}
&\quad\; F[P_C(\bm{\theta}_{k-1})] - F(\bm{\theta}_*) \\
&\le L \cdot \lVert P_C(\bm{\theta}_{k-1}) - \bm{\theta}_* \rVert^2 + \mathbb{E}\{\lVert \nabla f(\bm{\theta}_*; \bm{z}_{\xi_k}) \rVert \cdot \lVert P_C(\bm{\theta}_{k-1}) - \bm{\theta}_* \rVert \mid \mathcal{F}_{k-1}\} + \frac{1}{\rho_k} \mathbb{E}\{ \lVert \nabla f[P_C(\bm{\theta}_{k-1}); \bm{z}_{\xi_k}] \rVert^2 \mid \mathcal{F}_{k-1}\},
\end{split}
\end{equation*}
and taking expectations over $\mathcal{F}_{k-1}$:
\begin{equation*}
\begin{split}
&\quad\; \mathbb{E}\{ F[P_C(\bm{\theta}_{k-1})] - F(\bm{\theta}_*) \} \\
&\le L \cdot \mathbb{E}\{ \lVert P_C(\bm{\theta}_{k-1}) - \bm{\theta}_* \rVert^2 \} + \mathbb{E}\{\lVert \nabla f(\bm{\theta}_*; \bm{z}_{\xi_k}) \rVert \cdot \lVert P_C(\bm{\theta}_{k-1}) - \bm{\theta}_* \rVert \} + \frac{1}{\rho_k} \mathbb{E}\{ \lVert \nabla f[P_C(\bm{\theta}_{k-1}); \bm{z}_{\xi_k}] \rVert^2 \} \\
&= L \cdot \mathbb{E}\{ \lVert P_C(\bm{\theta}_{k-1}) - \bm{\theta}_* \rVert^2 \} + \mathbb{E}\{\lVert \nabla f(\bm{\theta}_*; \bm{z}_{\xi_k}) \rVert \} \cdot \mathbb{E}\{ \lVert P_C(\bm{\theta}_{k-1}) - \bm{\theta}_* \rVert \} \\
& \qquad + \frac{1}{\rho_k} \mathbb{E}\{ \lVert \nabla f[P_C(\bm{\theta}_{k-1}); \bm{z}_{\xi_k}] \rVert^2 \} \qquad\text{(two independent terms)} \\
&\le L \cdot \mathbb{E}\{ \lVert P_C(\bm{\theta}_{k-1}) - \bm{\theta}_* \rVert^2 \} + G \cdot \mathbb{E}\{ \lVert P_C(\bm{\theta}_{k-1}) - \bm{\theta}_* \rVert \} + \frac{c^2}{\rho_k},
\end{split}
\end{equation*}
where the last inequality is due to Proposition \ref{prop_bound} and Jensen's inequality. Replacing the subscript $k-1$ by $k$ will lead to the result:
\begin{equation*}
\begin{split}
\zeta_k &= \mathbb{E}\{ F[P_C(\bm{\theta}_{k})] - F(\bm{\theta}_*) \} \\
&\le L \cdot \mathbb{E}\{ \lVert P_C(\bm{\theta}_{k}) - \bm{\theta}_* \rVert^2 \} + G \cdot \mathbb{E}\{ \lVert P_C(\bm{\theta}_{k}) - \bm{\theta}_* \rVert \} + \frac{c^2}{\rho_{k+1}} \\
&\le L\delta_k + G\sqrt{\delta_k} + \frac{c^2}{\rho_1 } (k+1)^{-\gamma}\qquad\text{(by Theorem \ref{finite_para} and Jensen's inequality)} \\
&\le G\sqrt{\frac{2m(1 + \frac{\mu}{\rho_1})}{\mu \rho_1}}k^{-\frac{\gamma}{2}} + G (1+\frac{\mu}{\rho_1})^{-\frac{\phi_\gamma (k)}{2} } \sqrt{\delta_0 + A} \\
&\qquad + \frac{\mu c^2 + 2Lm(1 + \frac{\mu}{\rho_1})}{\mu \rho_1} k^{-\gamma}  + L \cdot (1+\frac{\mu}{\rho_1})^{-\phi_\gamma (k) }(\delta_0 + A).
\end{split}
\end{equation*}
In the last step, we combined the results in Theorem \ref{finite_para} and the fact that $\sqrt{a+b} \le \sqrt{a} + \sqrt{b}$ to bound $\sqrt{\delta_k}$, we also rewrite $\frac{c^2}{\rho_{k+1}} = \frac{c^2}{\rho_1}(k+1)^{-\gamma} \le \frac{c^2}{\rho_1}k^{-\gamma}$ to make the expression more concise without loss of rates.
\end{proof}

\section{Derivation of Proximal Distance Updates}
\subsection{Sparsity/Unit Ball Constrained Linear Regression}
Denote the whole data set as $\bm{X} = (\bm{x}_1,...,\bm{x}_n)^T$ and $\bm{y}=(y_1,...,y_n)^T$. Maximizing the log-likelihood under the sparsity constraint leads to the objective:
\begin{equation*}
\begin{split}
\min_{\bm{\theta}\in C} \; F(\bm{\theta}) &= \frac{1}{2n} \sum_{i=1}^n (y_i - \bm{x}_i^T \bm{\theta})^2,
\end{split}
\end{equation*}
then the gradient of $F(\bm{\theta})$ is:
\begin{equation*}
\begin{split}
 \nabla F(\bm{\theta})&= \frac{1}{n} \sum_{i=1}^n -(y_i - \bm{x}_i^T \bm{\theta})\bm{x}_i \\
    &= -\frac{1}{n}(\bm{X}^T\bm{y} - \bm{X}^T\bm{X}\bm{\theta}).
\end{split}
\end{equation*}
For the stochastic proximal distance algorithm, the gradient at a subset $\{\tilde{\bm{x}}_i,\tilde{y}_i\}_1^b$ is:
\begin{equation*}
\begin{split}
\frac{1}{b} \sum_{i=1}^b -(\tilde{y_i} - \tilde{\bm{x}}_i^T \bm{\theta})\tilde{\bm{x}}_i = -\frac{1}{b}(\tilde{\bm{X}}^T\tilde{\bm{y}} - \tilde{\bm{X}}^T\tilde{\bm{X}}\bm{\theta}),
\end{split}
\end{equation*}
then the proximal update is:
\begin{equation*}
    \bm{\theta}_k = P_C(\bm{\theta}_{k-1}) + \frac{1}{\rho_k}\frac{1}{b}(\tilde{\bm{X}}^T\tilde{\bm{y}} - \tilde{\bm{X}}^T\tilde{\bm{X}}\bm{\theta}_k),
\end{equation*}
and it has a closed form solution:
\begin{equation*}
\begin{split}
    \bm{\theta}_k &= (b\rho_k \bm{I}_p + \tilde{\bm{X}}^T\tilde{\bm{X}})^{-1}[b\rho_k P_C(\bm{\theta}_{k-1}) + \tilde{\bm{X}}^T\tilde{\bm{y}}] \\
    &= [\bm{I}_p - \tilde{\bm{X}}^T(b\rho_k \bm{I}_b + \tilde{\bm{X}}\tilde{\bm{X}}^T)^{-1}\tilde{\bm{X}}][P_C(\bm{\theta}_{k-1}) + (b\rho_k)^{-1}\tilde{\bm{X}}^T\tilde{\bm{y}}].
\end{split}
\end{equation*}
The second equation follows by invoking the Woodbury formula, and significantly reduces computations in calculating the inverse when $b \ll p$.

\subsection{Sparsity/Unit Ball Constrained Logistic Regression}
For the whole data set $\bm{X} = (\bm{x}_1,...,\bm{x}_n)^T$ and $\bm{y}=(y_1,...,y_n)^T$, maximizing the log-likelihood under the sparsity constraint leads to the objective:
\begin{equation*}
\begin{split}
\min_{\bm{\theta}\in C}\; F(\bm{\theta}) = -\frac{1}{n}\sum_{i=1}^n\{y_i(\bm{x}_i^T\bm{\theta}) - log(1+e^{\bm{x}_i^T\bm{\theta}})\}.
\end{split}
\end{equation*}
For the stochastic proximal distance algorithm, denote the loss on the subset $\{\tilde{\bm{x}}_i,\tilde{y}_i\}_1^b$ at iteration $k$ as
\begin{equation*}
\begin{split}
g_k(\bm{\theta}) &= -\frac{1}{b}\sum_{i=1}^b\{\tilde{y}_i(\tilde{\bm{x}}_i^T\bm{\theta}) - log(1+e^{\tilde{\bm{x}}_i^T\bm{\theta}})\},
\end{split}
\end{equation*}
then the proximal update is:
\begin{equation}\label{slr_updatingrule}
\begin{split}
\bm{\theta}_k = \argmin_{\bm{\theta}}\; g_k(\bm{\theta}) + \frac{\rho_k}{2}||\bm{\theta} - P_C(\bm{\theta}_{k-1})||^2.
\end{split}
\end{equation}
Since a closed form solution is not available for (\ref{slr_updatingrule}), we follow \citep{NIPS2017_061412e4} and employ a Newton iteration to solve this proximal operation, the iterative algorithm is: 
\begin{equation*}
\begin{split}
&\bm{\beta}_0 = \bm{\theta}_{k-1} \\
&\text{For $i=1,...,S$ Newton steps:} \\
&\quad\quad \bm{\beta}_i = \bm{\beta}_{i-1} - \eta * [\rho_k I + \nabla^2 g_k(\bm{\beta}_{i-1})]^{-1} \{\nabla g_k(\bm{\beta}_{i-1}) + \rho_k[\bm{\beta}_{i-1} - P_C(\bm{\theta}_{k-1})]\}\\
&\bm{\theta}_k = \bm{\beta}_S.
\end{split}
\end{equation*}
Here $\eta$ is selected by Armijo backtracking; the gradient and second derivative of $g_k$ are:
\begin{equation*}
\begin{split}
\nabla g_k(\bm{\theta}) &= -\frac{1}{b}\sum_{i=1}^b \{\tilde{\bm{x}}_i[\tilde{y}_i - \frac{e^{\tilde{\bm{x}}_i^T\bm{\theta}}}{1+e^{\tilde{\bm{x}}_i^T\bm{\theta}}}]\} \\
&= -\frac{1}{b} \tilde{\bm{X}}^T(\tilde{\bm{y}} - \tilde{\bm{p}}), \\
\nabla^2 g_k(\bm{\theta}) &= -\frac{1}{b}\sum_{i=1}^b\{-\tilde{\bm{x}}_i \frac{e^{\tilde{\bm{x}}_i^T\bm{\theta}}}{(1+e^{\tilde{\bm{x}}_i^T\bm{\theta}})^2} \tilde{\bm{x}}_i^T\} \\
&= \frac{1}{b}\tilde{\bm{X}}^T\tilde{\bm{W}}\tilde{\bm{X}},
\end{split}
\end{equation*}
where 
\begin{equation*}
\tilde{\bm{X}} = 
\begin{pmatrix}
\tilde{\bm{x}}_1\\
... \\
\tilde{\bm{x}}_b
\end{pmatrix},
\tilde{\bm{y}} = 
\begin{pmatrix}
\tilde{y}_1\\
... \\
\tilde{y}_b
\end{pmatrix},
\tilde{\bm{p}} = 
\begin{pmatrix}
\frac{e^{\tilde{\bm{x}}_1^T\bm{\theta}}}{1+e^{\tilde{\bm{x}}_1^T\bm{\theta}}}\\
... \\
\frac{e^{\tilde{\bm{x}}_b^T\bm{\theta}}}{1+e^{\tilde{\bm{x}}_b^T\bm{\theta}}}
\end{pmatrix},
\tilde{\bm{W}} = 
\begin{pmatrix}
\frac{e^{\tilde{\bm{x}}_1^T\bm{\theta}}}{(1+e^{\tilde{\bm{x}}_1^T\bm{\theta}})^2} &  & \\
&\ddots& \\
& & \frac{e^{\tilde{\bm{x}}_b^T\bm{\theta}}}{(1+e^{\tilde{\bm{x}}_b^T\bm{\theta}})^2}
\end{pmatrix}.
\end{equation*}
Again, when the batch size $b \ll p$, we can save computation by the Woodbury formula:
\begin{equation*}
\begin{split}
(\rho_k \bm{I}_p + \nabla^2 g_k(\bm{\theta}))^{-1} &= (\rho_k \bm{I}_p + \frac{1}{b}\tilde{\bm{X}}^T\tilde{\bm{W}}\tilde{\bm{X}})^{-1} \\
&= \rho_k^{-1} \bm{I}_p - \rho_k^{-1}\tilde{\bm{X}}^T(b\rho_k \bm{I}_b + \tilde{\bm{W}}\tilde{\bm{X}}\tilde{\bm{X}}^T)^{-1}\tilde{\bm{W}}\tilde{\bm{X}} .
\end{split}
\end{equation*}

Furthermore, according to \citet{toulis2014statistical}, the proximal operation (\ref{slr_updatingrule}) can be simplified to a one-dimensional fixed-point problem when the batch size $b = 1$.

\subsection{Sparsity/Unit Ball Constrained Huber Regression}
Denote the mini-batch of data sampled at iteration $k$ as $\{\tilde{\bm{x}}_i, \tilde{y}_i\}$, then the proximal mapping is to solve
\begin{equation}\label{huber_mmobj}
    \min_\theta\; \frac{1}{b}\sum_{i=1}^b L_{\delta}(\tilde{y}_i - \tilde{\bm{x}}_i^T \bm{\theta}) + \frac{\rho_k}{2} \lVert \bm{\theta} - P_C(\bm{\theta}_{k-1}) \rVert^2 := G(\bm{\theta})
\end{equation}

It can be proved by definition that $L_{\delta}(\tilde{y}_i - \tilde{\bm{x}}_i^T \bm{\theta})$ is convex in $\theta$, then so is $G(\theta)$. Therefore, we can solve \eqref{huber_mmobj} by gradient descent, the gradient is given by
\begin{equation*}
\begin{split}
\nabla G(\theta) &= \frac{1}{b}\sum_{i=1}^b L'_{\delta}(\tilde{y}_i - \tilde{\bm{x}}_i^T \bm{\theta})(-\tilde{\bm{x}}_i) + \rho_k[\bm{\theta} - P_C(\bm{\theta}_{k-1})] \\
&= -\frac{1}{b}\tilde{\bm{X}}^T \tilde{\bm{l}} + \rho_k[\bm{\theta} - P_C(\bm{\theta}_{k-1})],
\end{split}
\end{equation*}
where 
$$
\tilde{\bm{X}} = 
\begin{pmatrix}
\tilde{\bm{x}}_1\\
... \\
\tilde{\bm{x}}_b
\end{pmatrix},
\tilde{\bm{l}} = 
\begin{pmatrix}
L'_\delta(\tilde{y}_1 - \tilde{\bm{x}}_1^T \bm{\theta})\\
... \\
L'_\delta(\tilde{y}_b - \tilde{\bm{x}}_b^T \bm{\theta})
\end{pmatrix}.
$$
In practice, calculation of $\tilde{\bm{l}}$ can be vectorized, as $L'_\delta(a) = \min\{\lvert a \rvert, \delta\} \cdot \text{sign}(a)$.

\section{Data Generation and Additional Empirical Results} 
In the linear regression, the logistic regression, and the Huber regression models, we generate $n$ independent data $\{y_i,\bm{x}_i\}_{i=1}^n$, where $y_i$ is the response, $\bm{x}_i = (x_{i1}, ..., x_{ip})^T\in \mathbb{R}^p$ is the predictor, and $y_i$ depends on $\bm{x}_i$ through $\bm{\theta}_{\text{true}}$, the true regression coefficients. Under the unit ball constraint, $\lVert \bm{\theta}_{\text{true}} \rVert_2 = 2$ and $\bm{\theta}_{\text{true}}$ is specified by generating each element from the uniform distribution in $(-7, -4)\cup (4,7)$ and then scaling the vector. Under the sparsity constraint, $\lVert \bm{\theta}_{\text{true}} \rVert_0 = s \ll p$ and $\bm{\theta}_{\text{true}}$ is specified by generating $s$ i.i.d. random variables from the uniform distribution in $(-7, -4)\cup (4,7)$ and randomly assigning them to $5$ elements of $\mathbf{0}_p$. 

\paragraph{Linear Case} The data is generated by $x_{ij}\stackrel{i.i.d.}{\sim} N(0,1)$, and $y_i|\bm{x}_i\sim N(\bm{x}_i^T\bm{\theta}_{\text{true}}, 1)$. 

\paragraph{Logistic Case} The data is generated by $x_{ij}\stackrel{i.i.d.}{\sim} 0.3\cdot N(0, 1)$, and $y_i|\bm{x}_i\sim \text{Bernoulli}(\frac{1}{1+\text{exp}\{-\bm{x}_i^T\bm{\theta}_{\text{true}}\}})$.

\paragraph{Huber Case} The data is generated using the same way as in the linear model case, but we add extra i.i.d. noise, which follow the uniform distribution in $(-10, -5)\cup (5, 10)$, to $10\%$ randomly selected observations to make them outliers.

\paragraph{Matrix Regression}
In the matrix regression case, $\bm{\Theta}_{true}$ has $128$ elements equal to $1$, and the rest $3968$ elements equal to $0$. The non-zero elements are placed to ensure the low-rank structure. $\bm{X}_i \in \mathbb{R}^{p\times p}$, and its elements are generated i.i.d. from $N(0, 1)$, and $y_i\mid \bm{X}_i \sim N(\langle \bm{X}_i, \bm{\Theta} \rangle, 1)$.

\paragraph{Qualitative Behavior and Interaction Between $\rho_1$, $\gamma$} We further conducted experiments to explore the effect of the initial penalty $\rho_1$ and its interaction with $\gamma$. Indeed, to illustrate the tradeoff, consider the case $\gamma = 1$, where the convergence rate depends on the relationship between $\log(1+\frac{\mu}{\rho_1})$ and $1$. Analogously to the previous experiments, we fixed $\gamma = 1$ and consider different $\rho_1$, the analogous results of Figure 1 in the main paper are shown in Figure \ref{ub_fixgamma}. These results show that increasing $\rho_1$ tends to accelerate the convergence up to a point, but adversely affects convergence if it becomes too large. This behavior is reminiscent of step size selection in general for gradient-based methods.

The next set of experiments explores changes in the qualitative behavior of the algorithms over many settings of $\rho_1$ and $\gamma$. We run the method for $10,000$ iterations and report results averaged over $50$ restarts in each setting. Box plots of log-transformed errors across all the settings appear in Figure \ref{box_ub} and \ref{box_spar}. As we increase the initial size $\rho_1$, the performance improves up to a point, after which the algorithm begins to fail.   
This behavior is in line with the intuition behind increasing the penalties: as $\rho_k$ increases, the focus of the algorithm gradually shifts from learning from the data to obeying the constraint. This works best when the transition occurs gradually, a heuristic related to $\gamma$ which our theoretical results make rigorous. When $\rho_1$ is too large, the learning is too slow from the beginning, so that the algorithm quickly falls within the constraint but at a point that does not fit the data well. We also see that around this transition (for instance $\rho_1 = 10, \gamma = 5/3$ in the unit ball constrained low dimensional linear regression), cases where $\gamma$ lies outside the guaranteed range may fail to converge instead of speed up progress---this can be understood as switching from the data-driven to constraint-driven setting too abruptly or prematurely. These trade-offs motivate future theoretical work on characterizing convergence beyond the sufficient conditions we study, and also give practical guidance in line with intuitions behind the design of the algorithm.

\begin{figure*}[!h]
    \centering
    \includegraphics[width=1\textwidth]{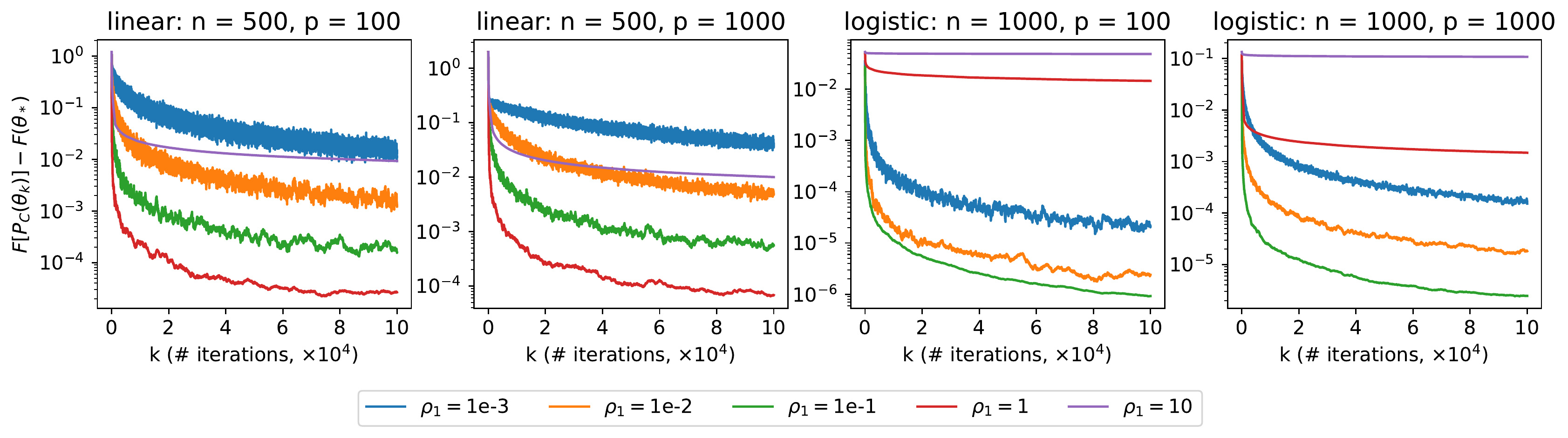}
    \includegraphics[width=1\textwidth]{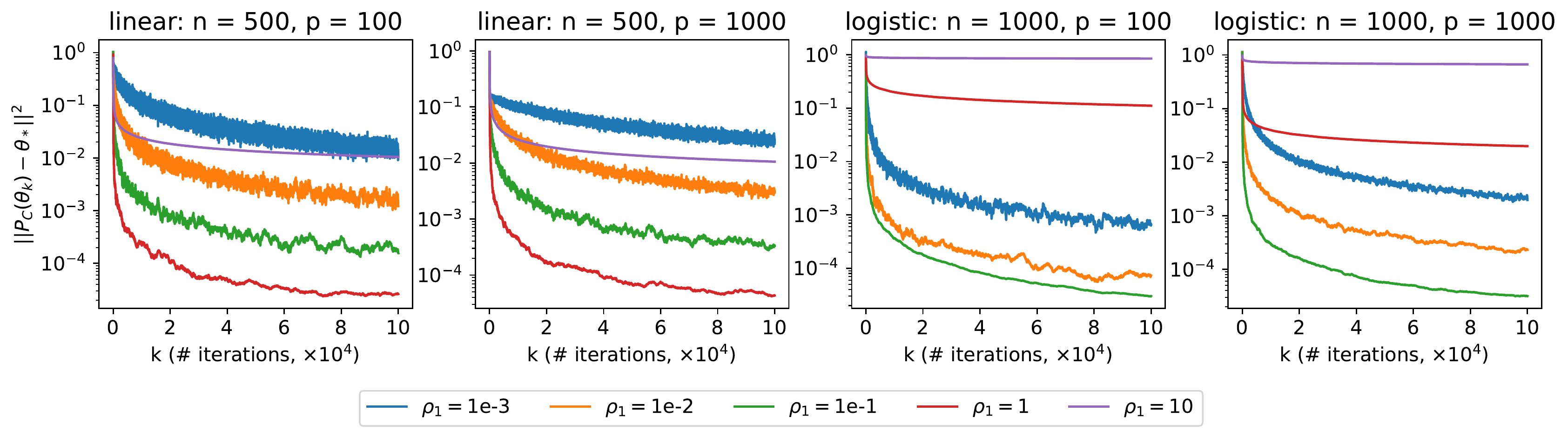}
    \caption{Finite error at different $\rho_1$ in different settings under the unit ball constraint}
    \label{ub_fixgamma}
\end{figure*}

Figure \ref{box_ub} presents the box plots of error at different combinations of $\gamma \in \{\frac{1}{3}, \frac{2}{3}, 1, \frac{4}{3}, \frac{5}{3} \}$ and $\rho_1 \in \{10^{-4}, 10^{-3}, 10^{-2}, 10^{-1}, 1, 10\}$ under unit ball constrained linear/logistic regression of various dimensions, with $50$ repeated experiments at each setting. Figure \ref{box_spar} presents the results under the sparsity constraint.

\begin{figure}[!h]
    \centering
  \begin{subfigure}{0.49\textwidth}
    \includegraphics[width=\textwidth]{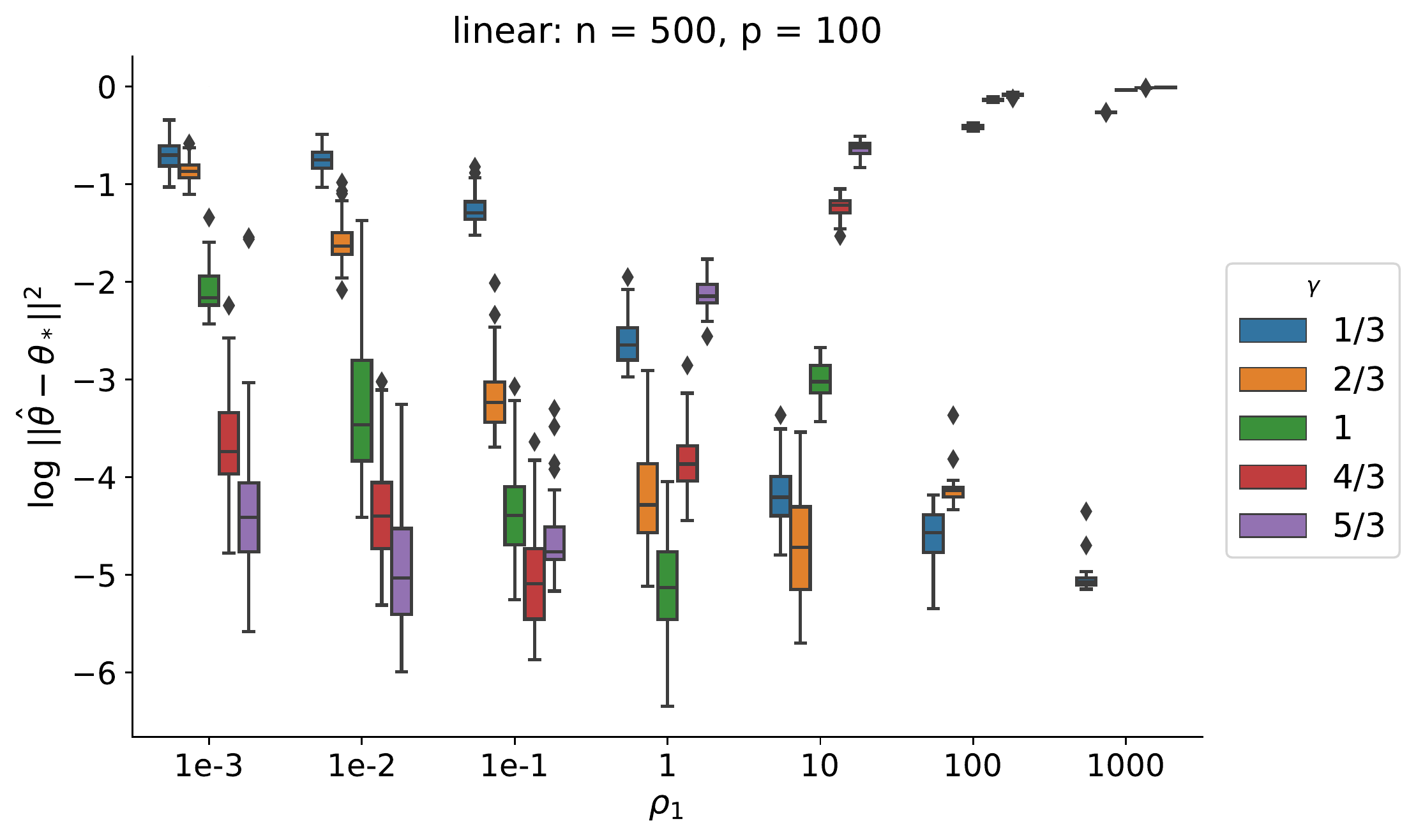}
  \end{subfigure}
  \begin{subfigure}{0.49\textwidth}
    \includegraphics[width=\textwidth]{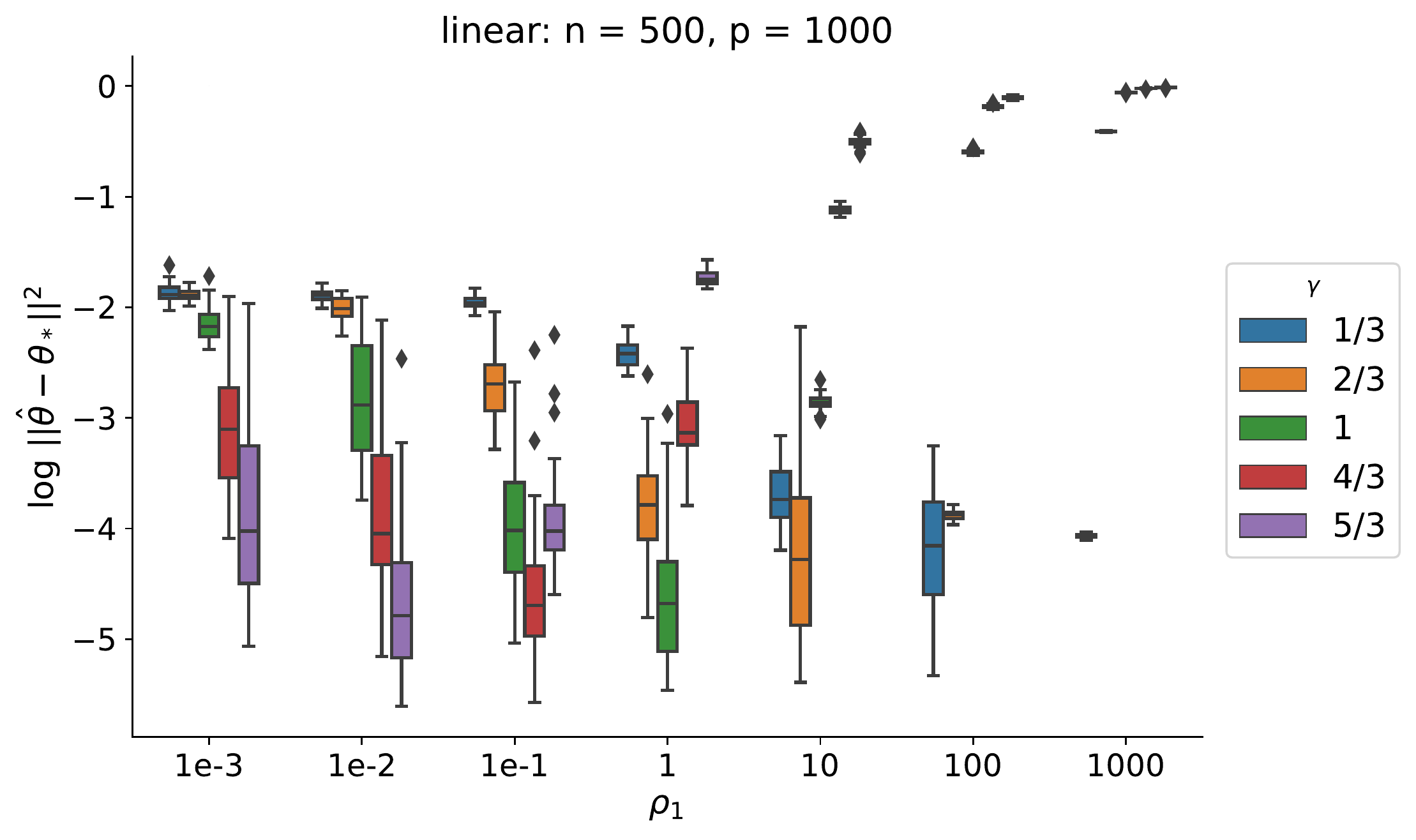}
  \end{subfigure}
  \begin{subfigure}{0.49\textwidth}
    \includegraphics[width=\textwidth]{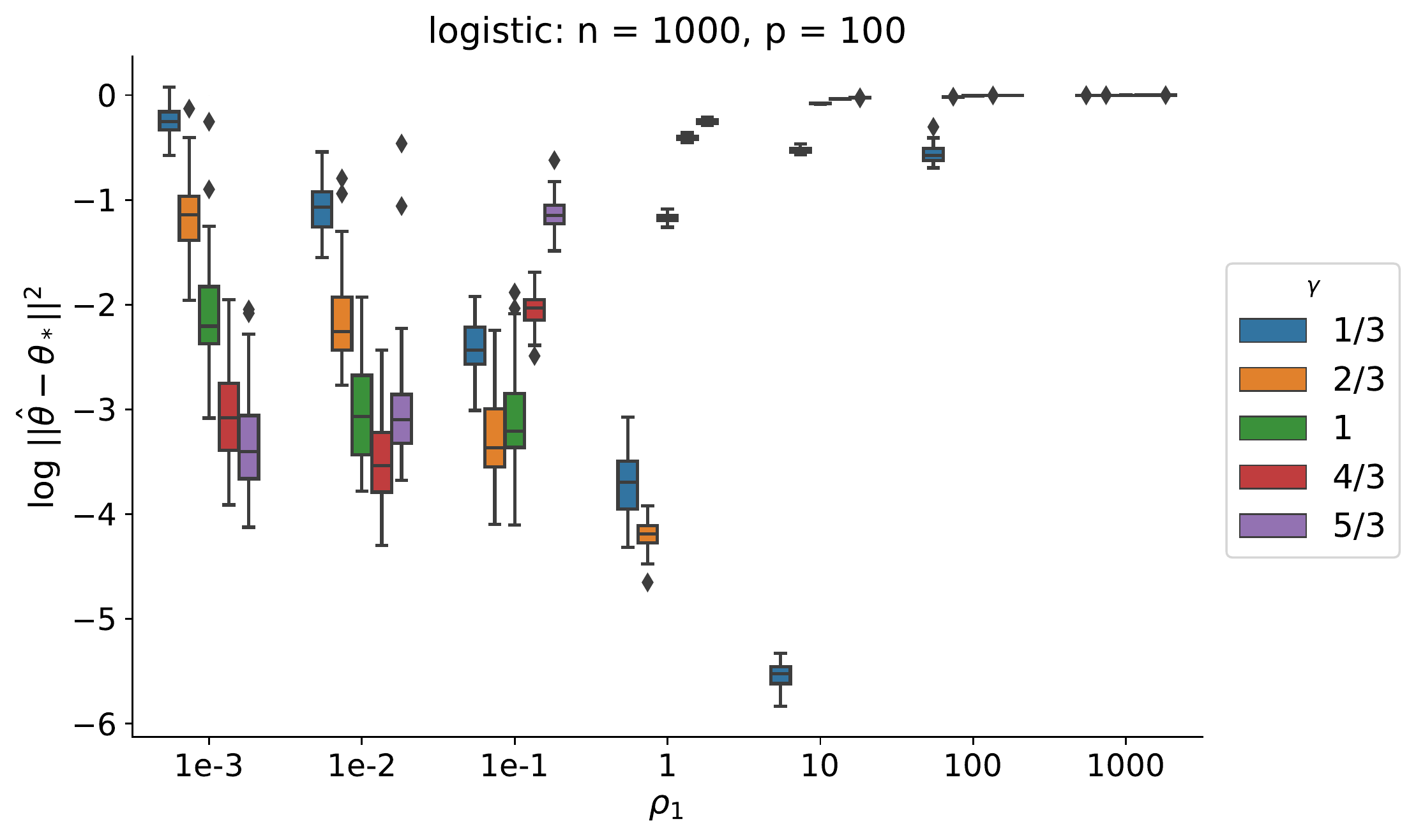}
  \end{subfigure}
  \begin{subfigure}{0.49\textwidth}
    \includegraphics[width=\textwidth]{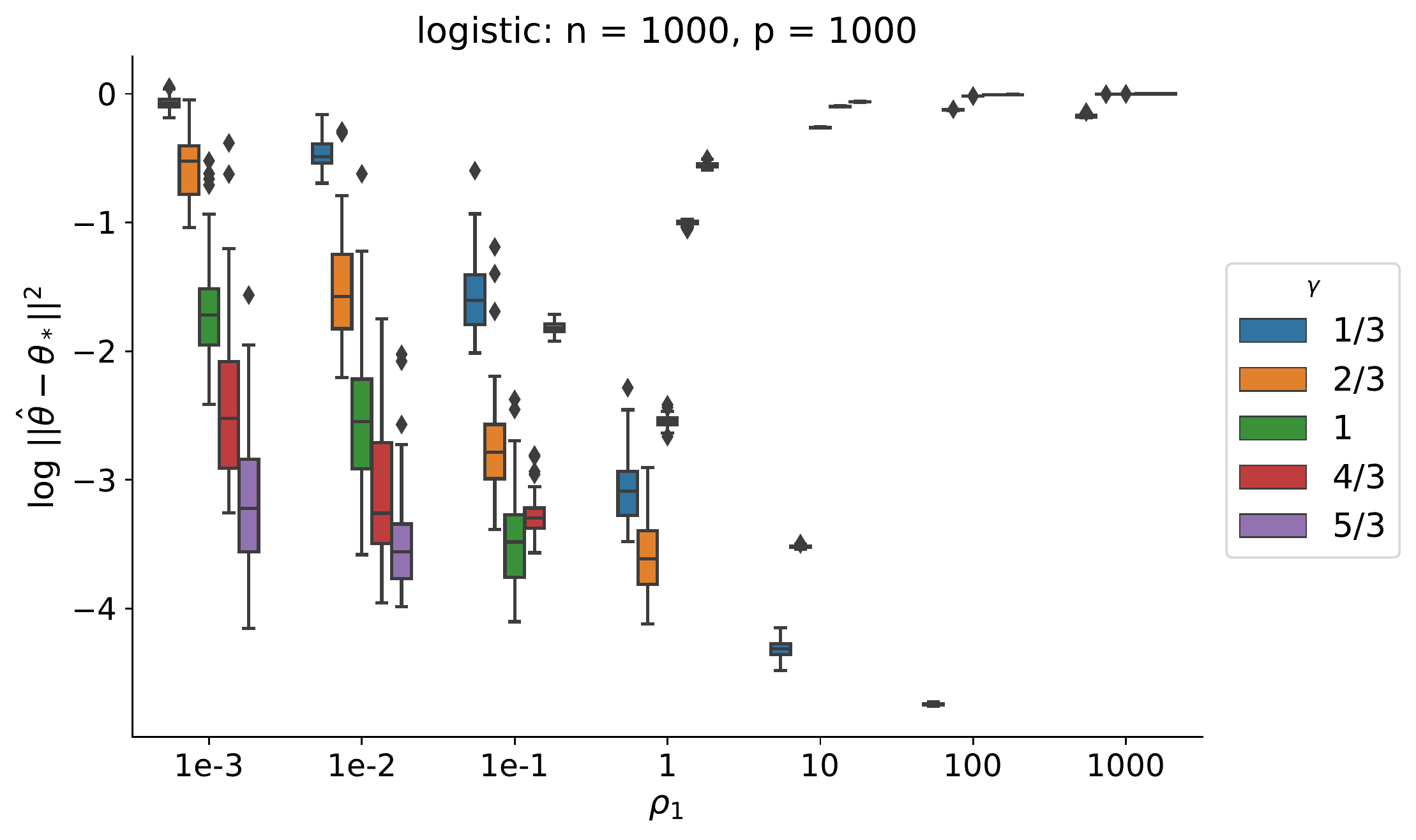}
  \end{subfigure}
  \caption{Box plots of error in various settings under the unit ball constraint.}
  \label{box_ub}
\end{figure}

\begin{figure}[!h]
    \centering
  \begin{subfigure}{0.49\textwidth}
    \includegraphics[width=\textwidth]{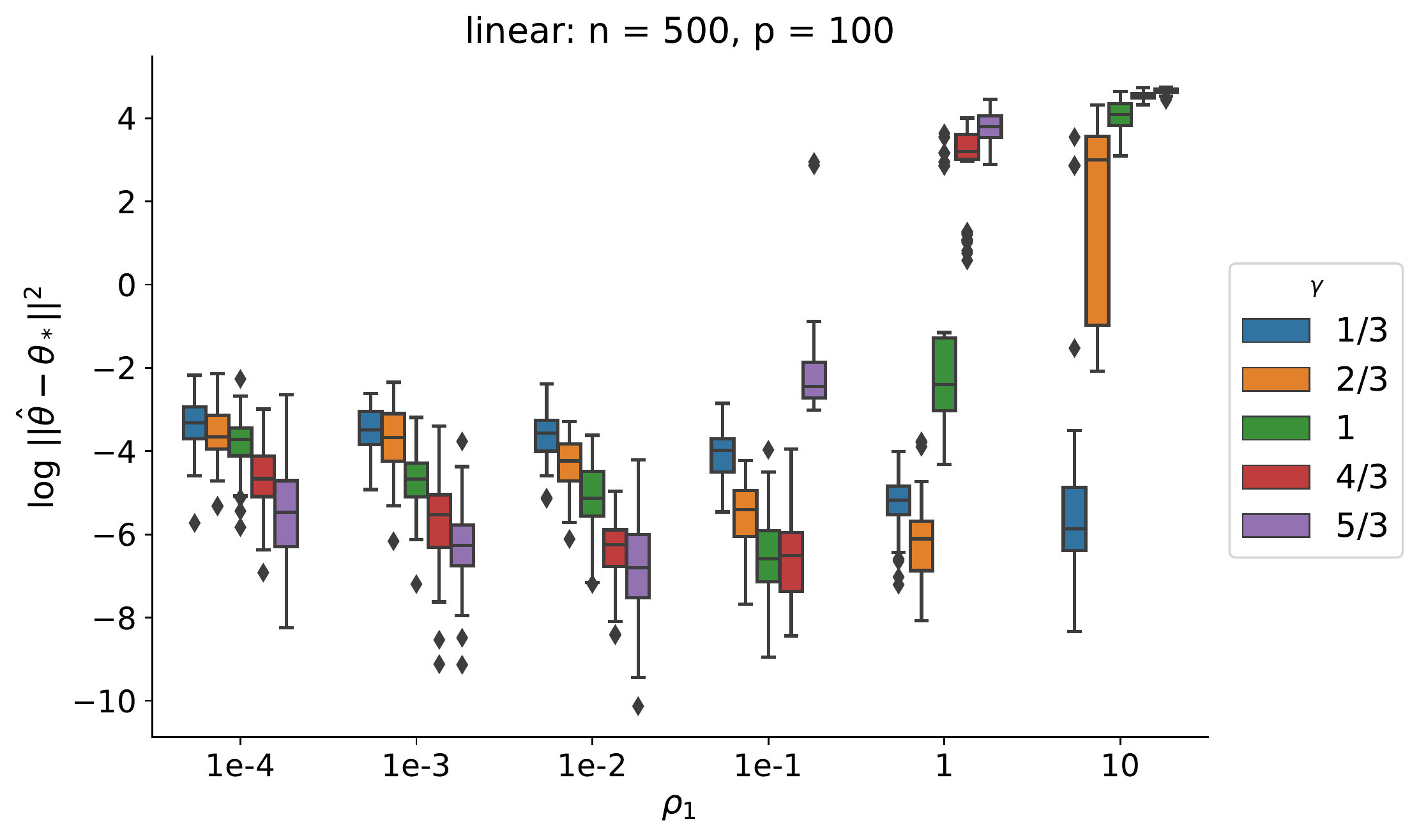}
  \end{subfigure}
  \begin{subfigure}{0.49\textwidth}
    \includegraphics[width=\textwidth]{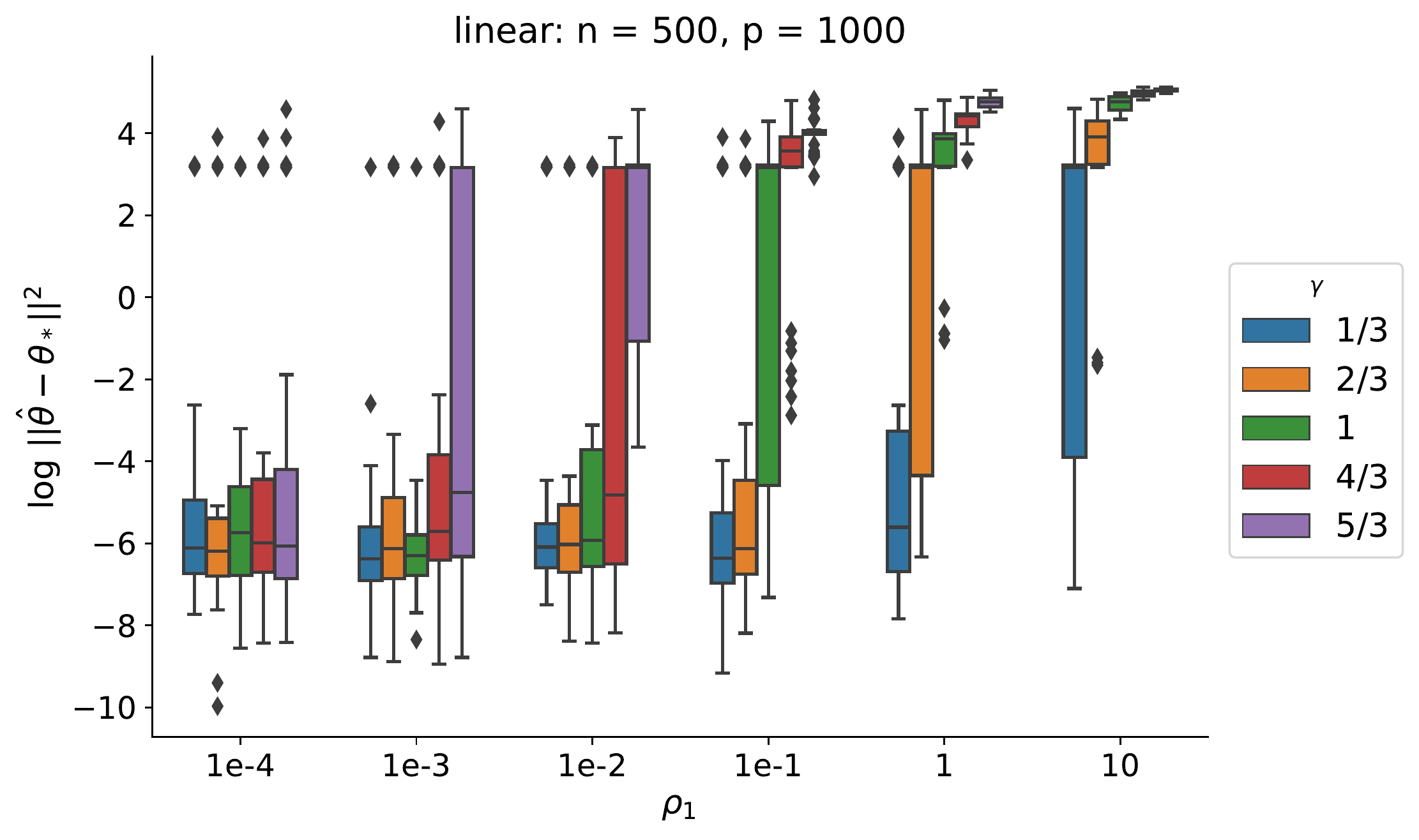}
  \end{subfigure}
  \begin{subfigure}{0.49\textwidth}
    \includegraphics[width=\textwidth]{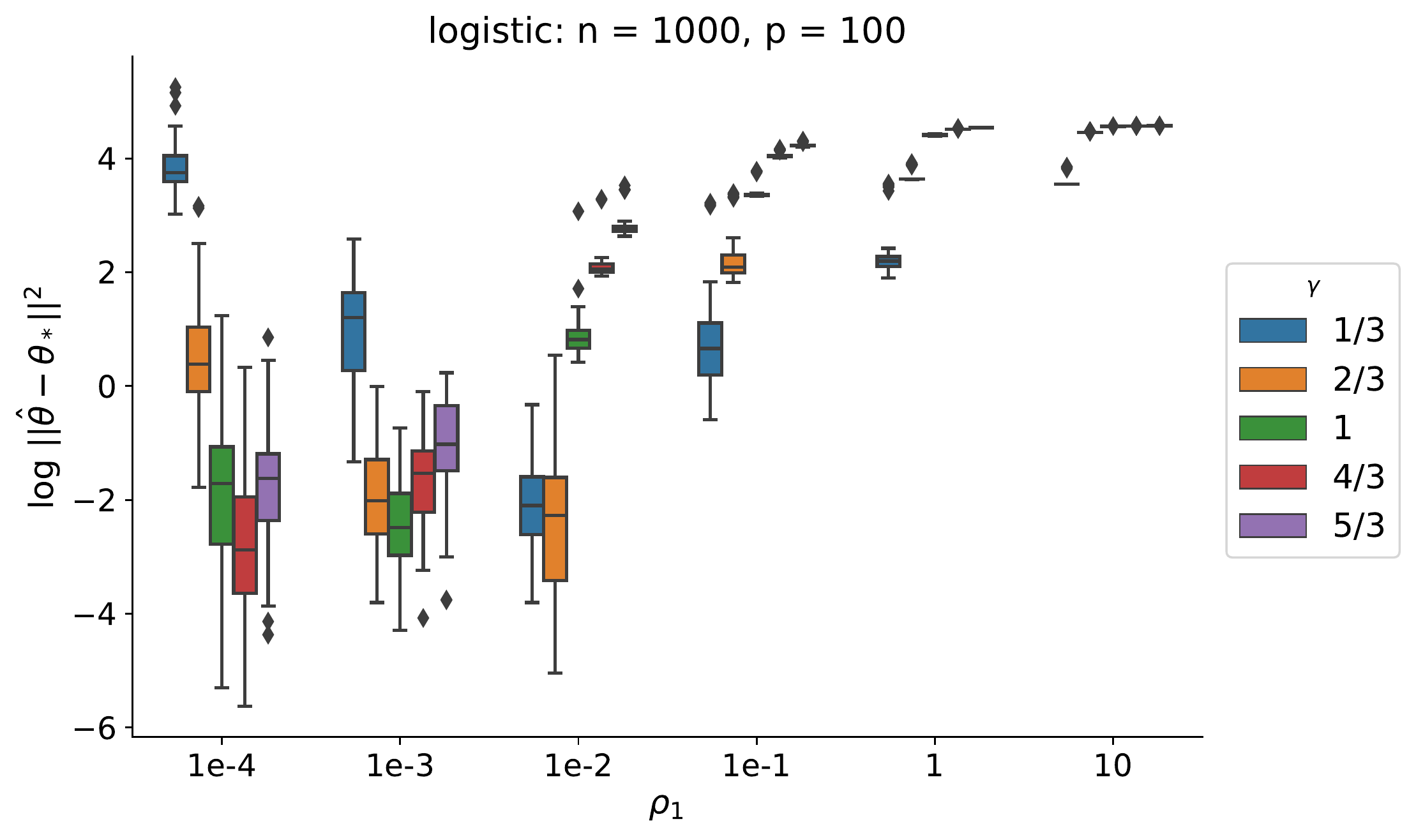}
  \end{subfigure}
  \begin{subfigure}{0.49\textwidth}
    \includegraphics[width=\textwidth]{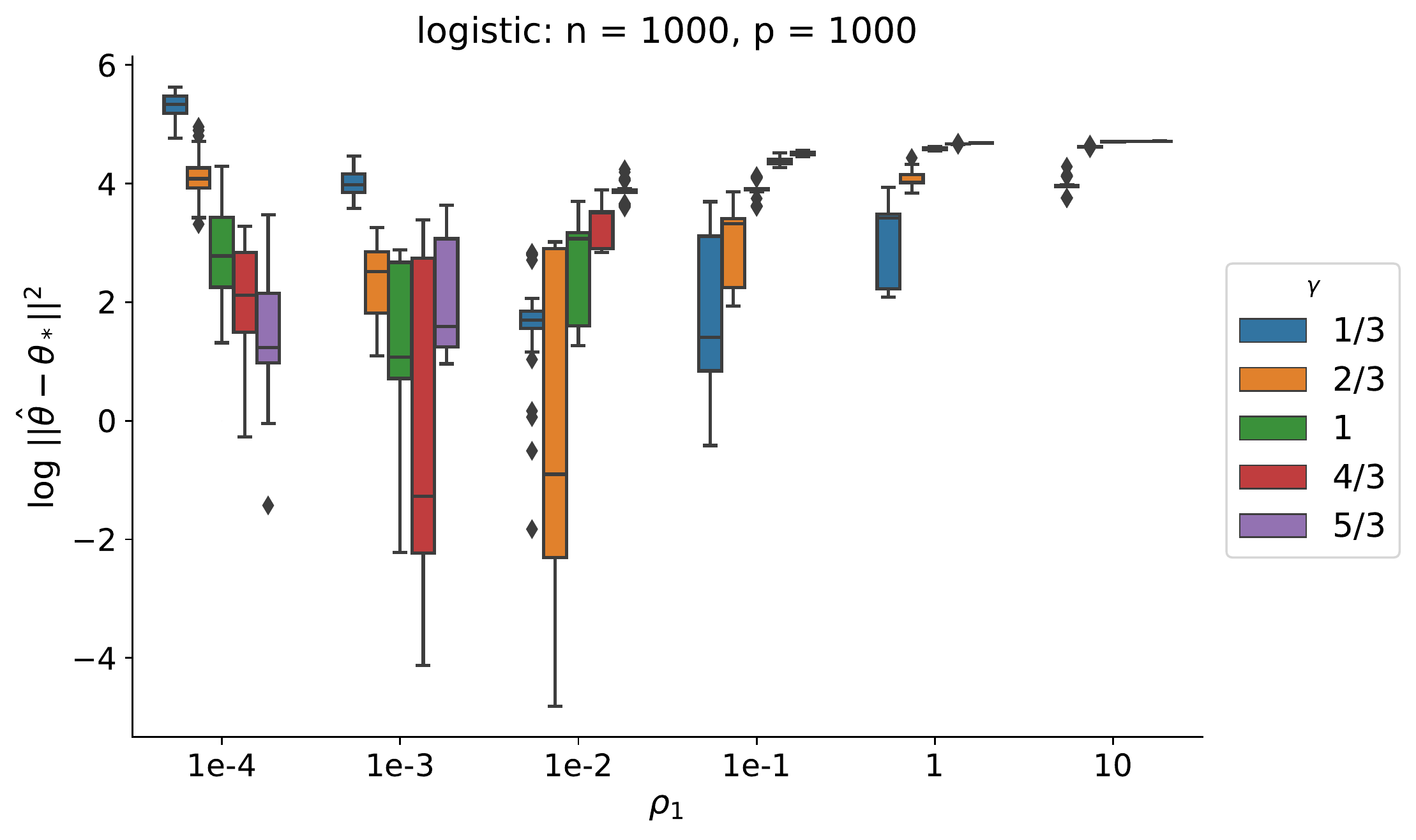}
  \end{subfigure}
  \caption{Box plots of error in various settings under the sparsity constraint. For reference to give a sense of scale, $\lVert \bm{\theta}_* \rVert \approx 12$ here.}
  \label{box_spar}
\end{figure}

\newpage
\paragraph{Supplemental Results of Real Data}
The results of runtime, number of iterations and accuracy are shown in Figure \ref{realdt_comp}, and the ROC plot under the batch size $500$ setting is shown in Figure \ref{realdt}.

\begin{figure}[h]
    \centering
\includegraphics[width=\textwidth]{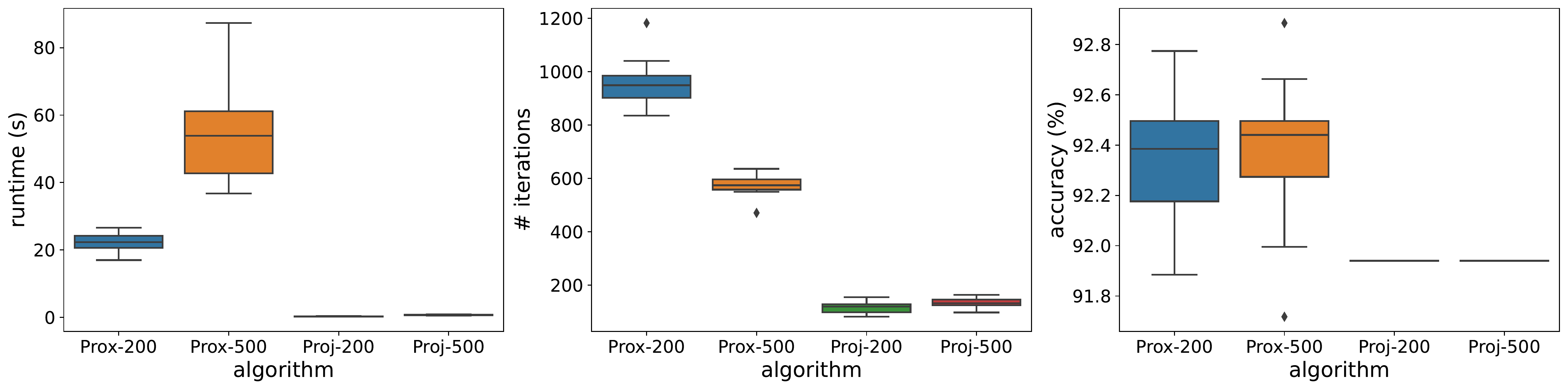}
    \caption{Detailed comparison of runtime, iterations until convergence, and classification accuracy on oral toxicity case study.}
    \label{realdt_comp} 
\end{figure}

\begin{figure}[h]
    \centering
\includegraphics[width=0.5\textwidth]{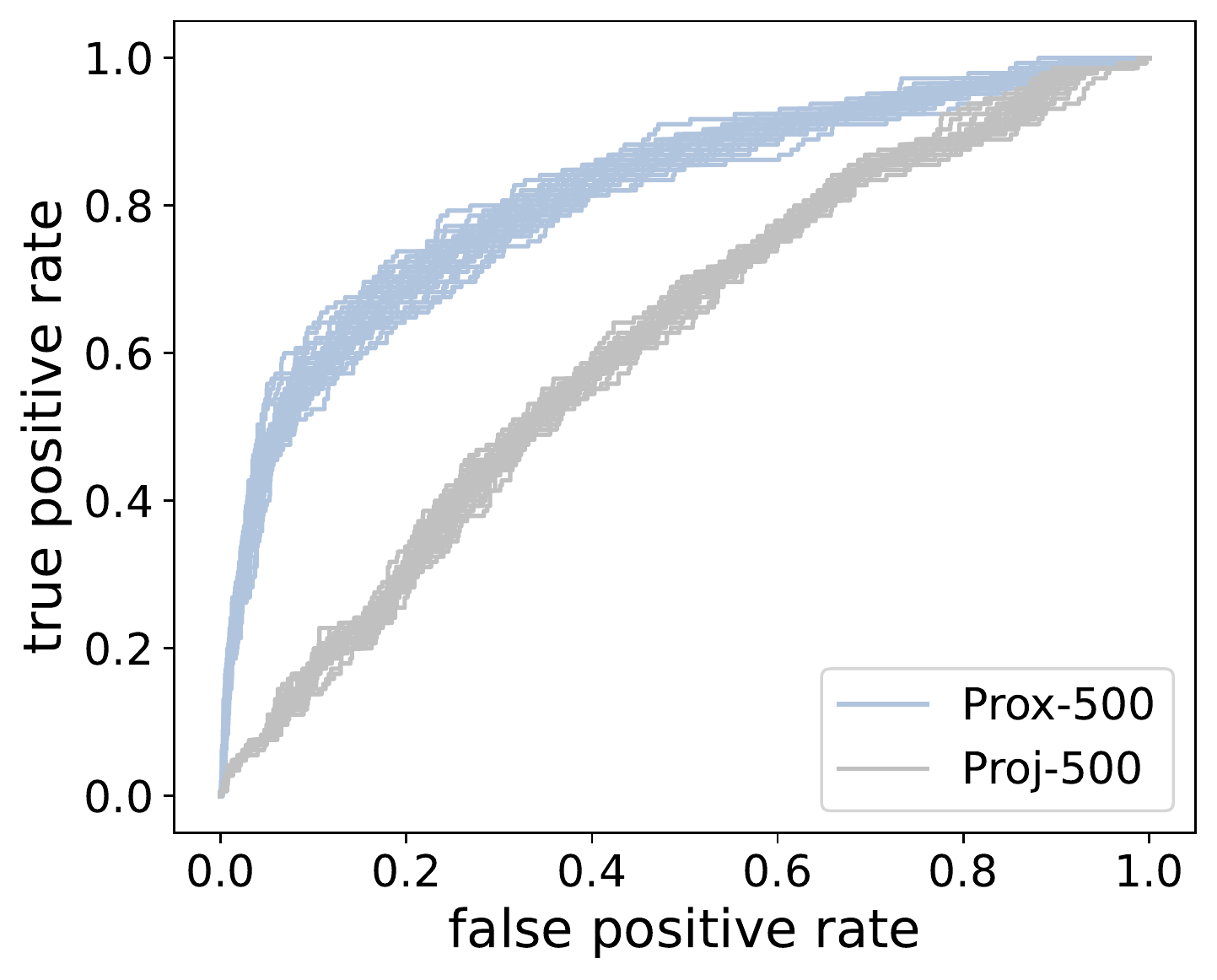}
    \caption{ROC comparison under batch size $500$, oral toxicity case study.}
    \label{realdt} 
\end{figure}

\end{document}